\pdfoutput=1
\documentclass[preprint,svgnames,dvipsnames]{elsarticle}

\usepackage[OT1]{fontenc}
\usepackage{amsmath,amsfonts, amssymb,amsthm}
\usepackage[numbers]{natbib}
\usepackage[colorlinks,citecolor=blue,urlcolor=blue]{hyperref}
\usepackage[left=3cm,right=3cm,top=3cm,bottom=3cm]{geometry}
\usepackage{algorithm}
\usepackage{algorithmicx}
\usepackage{algpseudocode}
\usepackage{bm}
\usepackage{graphicx}
\usepackage{dsfont}
\usepackage{tabularx}
\usepackage{stmaryrd}
\usepackage[svgnames,dvipsnames]{xcolor}
\usepackage{listings}

\usepackage[utf8]{inputenc} % allow utf-8 input
\usepackage[T1]{fontenc}    % use 8-bit T1 fonts
\usepackage{hyperref}       % hyperlinks
\usepackage{url}            % simple URL typesetting
\usepackage{booktabs}       % professional-quality tables
\usepackage{amsfonts}       % blackboard math symbols
\usepackage{nicefrac}       % compact symbols for 1/2, etc.
\usepackage{microtype}      % microtypography
\usepackage[dvipsnames]{xcolor}         % colors
\usepackage{algorithm}
\usepackage{algpseudocode}
\usepackage{multirow}
\usepackage{pifont}

\usepackage[american]{babel}
\usepackage{amsmath,amssymb,amsthm,mathtools,graphicx}
\usepackage{array,dcolumn}
\usepackage{dsfont}

\theoremstyle{plain}
\newtheorem{thm}{Theorem}
\newtheorem*{thm*}{Theorem}

\newtheorem{prop}{Proposition}[section]

\theoremstyle{definition}

\newtheorem*{exa}{Example}
\newtheorem*{rmk}{Remark} %Unnumbered

\newcommand{\wrt}{\text{w.r.t.}~}

\newcommand{\Tabref}[1]{Table~(\ref{#1})}
\newcommand{\Eqref}[1]{Eq.~(\ref{#1})}
\newcommand{\Figref}[1]{Figure~\ref{#1}}

\newcommand{\Secref}[1]{Section~\ref{#1}}
\newcommand{\Appendixref}[1]{Appendix~\ref{#1}}

\newcommand{\Propref}[1]{Proposition~\ref{#1}}

\newcommand{\calX}{\mathcal{X}}
\newcommand{\calY}{\mathcal{Y}}

\newcommand{\calU}{\mathcal{U}}
\newcommand{\calD}{\mathcal{D}}
\newcommand{\calP}{\mathcal{P}}

\newcommand{\calN}{\mathcal{N}}
\newcommand{\calS}{\mathcal{S}}

\newcommand{\R}{\mathbb{R}}

\newcommand{\A}{\mathbb{A}}

\newcommand{\Prob}{\mathbb{P}}

\newcommand{\lrpar}[1]{\left( #1 \right)}
\newcommand{\lrbra}[1]{\left[ #1 \right]}
\newcommand{\lrcubra}[1]{\left\{ #1 \right\}}

\makeatletter
\newcommand{\bigplus}{%
  \DOTSB\mathop{\mathpalette\mattos@bigplus\relax}\slimits@
}
\newcommand\mattos@bigplus[2]{%
  \vcenter{\hbox{%
    \sbox\z@{$#1\sum$}%
    \resizebox{!}{0.9\dimexpr\ht\z@+\dp\z@}{\raisebox{\depth}{$\m@th#1+$}}%
  }}%
  \vphantom{\sum}%
}
\makeatother

\DeclareFontFamily{U}{mathx}{}
\DeclareFontShape{U}{mathx}{m}{n}{<-> mathx10}{}
\DeclareSymbolFont{mathx}{U}{mathx}{m}{n}
\DeclareMathAccent{\widecheck}{0}{mathx}{"71}

\newcommand{\pset}[1]{\calP_{#1}}
\newcommand{\TrSet}{\calD^{\textrm{Tr}}}
\newcommand{\CalSet}{\calD^{\textrm{Cal}}}

\newcommand{\absval}[1]{\left| #1 \right|}

\newcommand{\eqdef}{:=}

\newcommand{\Shap}{\textrm{Shap}}
\newcommand{\PShap}{\textrm{P-Shap}}

\newcolumntype{L}[1]{>{\raggedright\let\newline\\\arraybackslash\hspace{0pt}}m{#1}}
\newcolumntype{C}[1]{>{\centering\let\newline\\\arraybackslash\hspace{0pt}}m{#1}}
\newcolumntype{R}[1]{>{\raggedleft\let\newline\\\arraybackslash\hspace{0pt}}m{#1}}

\makeatletter
\def\ps@pprintTitle{%
\let\@oddhead\@empty
\let\@evenhead\@empty
\def\@oddfoot{}%
\let\@evenfoot\@oddfoot}
\makeatother

\begin{document}

\begin{frontmatter}

%%%%%%%%%%%%%%%%%%%%%%%%%%%%%%%%
% Title & runtitle
\title{Unveil Sources of Uncertainty: Feature Contribution to Conformal Prediction Intervals}

%%%%%%%%%%%%%%%%%%%%%%%%%%%%%%%%
% Authors information

\author[a,b,e]{Marouane Il Idrissi}
\author[a]{Agathe~Fernandes Machado}
\author[c,d]{Ewen Gallic}
\author[a]{Arthur Charpentier}

\address[a]{Département de Mathématiques, Université du Québec à Montréal, Montréal, QC Canada}
\address[b]{Institut Intelligence et Données, Université Laval, Québec, QC Canada}
\address[c]{CNRS - Université de Montréal CRM -- CNRS}
\address[d]{Aix Marseille Univ, CNRS, AMSE, Marseille, France}
\address[e]{Corresponding Author - Email: ilidrissi.m@gmail.com}

%%%%%%%%%%%%%%%%%%%%%%%%%%%%%%%%%%%%%%%%%%%%%%%%%
% ABSTRACT
\begin{abstract}
    Cooperative game theory methods, notably Shapley values, have significantly enhanced machine learning (ML) interpretability. However, existing explainable AI (XAI) frameworks mainly attribute average model predictions, overlooking predictive uncertainty. This work addresses that gap by proposing a novel, model-agnostic uncertainty attribution (UA) method grounded in conformal prediction (CP). By defining cooperative games where CP interval properties—such as width and bounds—serve as value functions, we systematically attribute predictive uncertainty to input features. Extending beyond the traditional Shapley values, we use the richer class of Harsanyi allocations, and in particular the proportional Shapley values, which distribute attribution proportionally to feature importance. We propose a Monte Carlo approximation method with robust statistical guarantees to address computational feasibility, significantly improving runtime efficiency. Our comprehensive experiments on synthetic benchmarks and real-world datasets demonstrate the practical utility and interpretative depth of our approach. By combining cooperative game theory and conformal prediction, we offer a rigorous, flexible toolkit for understanding and communicating predictive uncertainty in high-stakes ML applications.
\end{abstract}

%%%%%%%%%%%%%%%%%%%%%%%%%%%%%%%%
% Keywords


\end{frontmatter}

%%===========================================================%%
%% 1. Introduction
\section{Introduction}
Over the past several years, the field of explainable artificial intelligence (XAI) has gained significant attention within machine‑learning (ML) research \cite{NEURIPS2024maze, NEURIPS2024_eb3a9313}. This interest is partly explained by the growing need to interpret and justify models in high‑stakes domains. For instance, in actuarial science, interpretability clarifies premium calculations for policyholders \cite{xin2024trust}; in healthcare, it ensures that diagnostic support systems rely on clinically meaningful features rather than spurious correlations \cite{pahde2025ensuringmedicalaisafety}; and in criminal‑risk assessment, it helps reduce historical biases and prevents sensitive attributes from influencing predictions \cite{chouldechova2017pp}.

When decision‑makers need to understand which features drive a model's output, the theory of cooperative games provides a resourceful framework: it enables a systematic methodology for feature attributions. The Shapley value is the most widely used allocation rule for attributing predictions to input features \cite{NIPS2017_8a20a862, causalshap, Lundberg2020XAI, JMLR:v22:20-1316}. Yet the Shapley values are only one member of a richer family of allocations. The Harsanyi set of allocations \cite{Vasilev2001} generalizes Shapley values and includes variants such as the weighted and proportional Shapley values. This class of allocations yields a flexible spectrum of importance measures rooted in different normative perspectives. Thus, it offers a unique playground for defining a wide range of importance measures for a plethora of different goals.

Cooperative game theoretic approaches are already prominent in XAI through tools like SHAP \cite{NIPS2017_8a20a862}, which decompose model predictions into feature‑wise contributions. While SHAP offers unparalleled insights into a model's behavior, other key aspects may be of interest to practitioners. For instance, quantifying predictive uncertainty has gained a lot of traction in modern applications, serving as an important tool for assessing the reliability of model predictions. This capability is particularly valuable in critical domains such as clinical diagnostics, autonomous driving, automated trading, and energy production. Conformal prediction (CP) fulfills this need by producing finite‑sample valid prediction intervals for virtually any ML model under broad assumptions on the data‑generating process \cite{vovkCP2008tutorial}. CP has demonstrated practical value in applications like environmental monitoring for climate policymaking, where it effectively mitigates the risk of overconfident models \cite{Singh2024}. The CP literature is quickly expanding, with many recent theoretical and practical advances.

XAI and uncertainty quantification have often been combined to manage risks in high-stakes settings \cite{Barredo2020, Razavi2021, Iooss2022, milEJS2024}, leading to uncertainty attribution (UA) methods. In ML, UA seeks to identify which features most influence a model's predictive uncertainty \cite{DBLPeccvWangBJ24}. Proxies for uncertainty include information‑theoretic quantities (e.g., entropy) \cite{Lundberg2020XAI, JMLR:v22:20-1316, pmlr-v80-chen18j, pmlr-v206-jethani23a, watson2023explaining, NEURIPS2020_c7bf0b7c}, predictive variance \cite{BLEY2025111171, iversen2024}, and counterfactual explanations \cite{antoran2021getting, ley2022aaai, pmlr-v180-perez22a}.

%%===========================================================%%
%% 1.1 Contributions
\subsection{Contributions}
Our work leverages feature attribution methods to deepen the understanding of predictive uncertainty. The main contributions are:

\textbf{CP-based model-agnostic UA method to decompose uncertainty}~ We introduce a novel regression-model-agnostic UA approach to decompose uncertainty, measured using CP-based quantities (e.g., width, midpoint, upper and lower values of the prediction intervals). This method effectively provides actionable insights into how each feature affects predictive uncertainty.

\textbf{Beyond Shapley: Harsanyi allocations and proportional Shapley values}~ Building on cooperative game theory, we employ the Harsanyi set of allocations and introduce the proportional Shapley values. Whereas classical Shapley values implement an egalitarian redistribution of dividends, proportional Shapley values allocate them in proportion to feature contributions. We use and compare both allocations within our UA framework.

\textbf{Statistically grounded approximations for faster computations}~ Attribution methods often suffer from high computational cost. Alongside an exact algorithm, we develop a Monte‑Carlo approximation scheme with provable statistical guarantees, granting explicit control over runtime.

\textbf{Extensive empirical evaluation}~ We validate our method through comprehensive experiments on simulated and real‑world datasets spanning diverse dimensions, sample sizes, and data types.

The proofs of the theoretical results presented in the article are postponed to Appendix~\ref{apdx:proofs}.
%%===========================================================%%
%% 1.2 Related work 
\subsection{Related work}

\textbf{Feature attributions within the CP framework}~ Several studies connect CP and XAI. Some quantify uncertainty in feature‑importance scores via CP \cite{pmlr-v204-alkhatib23a, pmlr-v230-alkhatib24a}, while others use Shapley values as conformity functions for CP intervals \cite{pmlr-v152-jaramillo21a}. To date, only \cite{Mehdiyev2024} directly attributes CP‑derived uncertainty, focusing on Shapley values with quantile‑regression forests; UA methods tailored to CP remain largely unexplored.

\textbf{Other intersections of XAI and CP}~ Beyond UA, substantial crossover exists between CP‑based uncertainty quantification and interpretable ML. Examples include oracle coaching combined with CP \cite{pmlr-v105-johansson19a}; \emph{conformal trees} that enforce homogeneous outputs across leaves \cite{gil2024uai, treeconform2019}; and rule‑based classifiers exploiting rule‑boundary geometry \cite{pmlr-v204-narteni23a}. Feature‑perturbation analyses extend to CP‑specific quantities: \textit{ConformaSight} explains the size and coverage of non‑adaptive CP sets via counterfactual perturbations \cite{Yapicioglu2024}; \cite{ce2024, pmlr-v230-lofstrom24a} generate factual and counterfactual explanations of calibrated CP intervals; and \cite{MEHDIYEV2025110363} apply permutation feature importance across training, calibration, and test sets for CP intervals in regression models.

%%===========================================================%%
%% 2. Conformal predictions and cooperative games
\section{Conformal predictions and cooperative games}
This section first reviews the \emph{split conformal prediction} (SCP) framework, also called \emph{inductive conformal prediction} in the literature \cite{Zaffran2024}, and then introduces the cooperative-game concepts required for our method.

\textbf{Notations}~ Let $\calD \eqdef \lrcubra{\lrpar{X_i, Y_i}}_{i=1}^n \in \lrpar{\calX \times \calY}^n$ be a sequence of \emph{exchangeable random variables} drawn from an unknown distribution $P_{X,Y}$, where $\calX \subseteq \R^d$ and $\calY \subseteq \R$. Let $(X_{n+1}, Y_{n+1}) \sim P_{X,Y}$ denote an additional data point outside $\calD$. Denote $D = \lrcubra{1,\dots, d}$ and let $\pset{D}$ be its power set (the set of all subsets of $D$). For each $A \in \pset{D}$, let $X_i^{(A)} \in \calX_A \subseteq \R^{|A|}$ denote the sub-vector that retains only the features whose indices lie in $A$; by convention, $X_i^{(\emptyset)} := \emptyset$ and $X_i^{(D)} := X_i$. Define $\calD_A \eqdef \lrcubra{\lrpar{X_i^{(A)}, Y_i}}_{i=1}^n$ . Splitting $\lrbra{n}$ at random into two disjoint parts yields a \emph{training set} $\TrSet \subset \calD$ (with $\TrSet_A \subset \calD_A$ for each $A$) and a \emph{calibration set} $\CalSet \subset \calD$ (with $\CalSet_A \subset \calD_A$ for each $A$).

%%================================%%
%% 2.1 SCP
\subsection{Split conformal prediction} \label{sec:scp}

Constructing an SCP interval requires two ingredients:  
i) a \emph{model} $\widehat{f} : \calX \rightarrow \calY$ fitted on $\TrSet$ by a learning algorithm $\A$;  
ii) a \emph{conformity score} $s : \calX \times \calY \times \calY^\calX \rightarrow \R$, used to form the set $\calS = \lrcubra{s(X_i, Y_i, \widehat{f}) : (X_i, Y_i) \in \CalSet}$. 
The resulting interval is
\begin{equation}
    \widehat{C}\lrpar{X_{n+1}} \eqdef \lrcubra{y \in \calY : \, s\lrpar{X_{n+1}, y, \widehat{f}} \leq q_{1-\alpha}(\calS)},
    \label{eq:predInterval}
\end{equation}
where $q_{1-\alpha}(\calS)$ is the empirical $(1-\alpha)$-quantile of $\calS$ for a chosen level $\alpha \in (0,1)$. Assuming no ties occur in $\calS$, SCP achieves marginal coverage \cite{Lei2018}:
\[
1 - \alpha \leq \Prob\lrpar{Y_{n+1} \in \widehat{C}\lrpar{X_{n+1}}} \leq 1 - \alpha + \frac{1}{\#\CalSet + 1}.
\]
Tie-breaking randomization can tighten this bound to equality \cite{Vovk2022}, but that refinement is not essential for our purposes. Different choices of $\widehat{f}$ and $s$ give rise to distinct CP variants, three of which are summarized below.

\textbf{Standard mean regression (SMR)}~ When $\widehat{f}$ is a mean-regression model, the SMR setting \cite{Zaffran2024} adopts the score $s\lrpar{x,y,\widehat{f}} = \absval{y-\widehat{f}(x)}$, yielding
\[
\widehat{C}\lrpar{X_{n+1}} = \lrbra{\widehat{f}(X_{n+1}) \pm q_{1-\alpha}(\calS)}.
\]
It is important to note that the interval width is constant in $X_{n+1}$. This lack of \emph{adaptability} to the newly observed datapoint motivates the adoption of a more adaptive conformity score.

\textbf{Local adaptive conformal prediction (LACP)}~ In \cite{Lei2018}, the authors proposed to scale the residuals by an estimate of local dispersion, using 
$s\lrpar{x,y,\lrpar{\widehat{f}, \widehat{\sigma}}} = \absval{y- \widehat{f}(x)}/\widehat{\sigma}(x)$. Here, $\widehat{f}$ and $\widehat{\sigma}$ are typically models of the conditional mean and conditional dispersion (e.g., standard deviation, mean absolute dispersion), respectively. The prediction intervals adopt the form
\[
\widehat{C}\lrpar{X_{n+1}} = \lrbra{\widehat{f}(X_{n+1}) \pm q_{1-\alpha}(\calS)\,\widehat{\sigma}(X_{n+1})}.
\]
\textbf{Conformalized quantile regression (CQR)}~ 
Proposed by \cite{Romano2019}, CQR differs by using models of lower and upper levels of conditional quantiles. These models are fitted on $\TrSet$, leading to $\widehat{f} = \bigl(\widehat{q}_{\textrm{low}},\widehat{q}_{\textrm{up}}\bigr)$, along with the conformity score $s\lrpar{x,y,\lrpar{\widehat{q}_{\textrm{low}},\widehat{q}_{\textrm{up}}}} = \max\!\bigl\{\widehat{q}_{\textrm{low}}(x)-y,\, y-\widehat{q}_{\textrm{up}}(x)\bigr\}$. The interval is then given by
\[
\widehat{C}\lrpar{X_{n+1}} = \lrbra{\widehat{q}_{\textrm{low}}(X_{n+1}) - q_{1-\alpha}(\calS),\, \widehat{q}_{\textrm{up}}(X_{n+1}) + q_{1-\alpha}(\calS)}.
\]
Although using $(\alpha/2,1-\alpha/2)$ is a standard choice for lower and upper quantile levels (i.e., for $(\widehat{q}_{\textrm{low}}, \widehat{q}_{\textrm{up}})$), other choices also preserve marginal validity \cite{Romano2019}.

%%================================%%
%% 2.2 Feature Attributions
\subsection{Cooperative game theory for feature attribution} \label{subsec:alloc}

A (transferable-utility) cooperative game is a pair $(D,v)$, where $D=\{1,\dots,d\}$ is a set of players and $v : \pset{D} \rightarrow \R$ is the \emph{value function}. The goal of the value function is to quantify the value of each coalition of players. An \emph{allocation} (a.k.a., solution concept or payoff) \cite{Osborne1994} is a mapping $\phi_v : D \rightarrow \R$. An allocation is \emph{efficient} if $\sum_{j \in D} \phi_v(j) = v(D) - v(\emptyset)$, i.e., it redistributes $v(D)$ among the players.

For any coalition $A \in \pset{D}$, the Harsanyi dividend \cite{Harsanyi1963} is defined as $\varphi_v(A) = \sum_{B \subseteq A} (-1)^{|A|-|B|}\,v(B)$ and represents the value created by the interaction among players in $A$. Here, $v$ can be interpreted as the cumulative value of the players, while $\varphi_v$ quantifies the added value of a coalition. The \emph{Harsanyi set of allocations} \cite{Vasilev2001} consists of all mappings of the form
\[
\phi_v(j) = \!\!\sum_{A \in \pset{D} : j \in A} \!\!\lambda_j(A)\,\varphi_v(A),
\quad\text{with}\quad \lambda_j(A) \ge 0, ~ \sum_{j \in D}\lambda_j(A) = 1 ~ \text{and}~ \lambda_j(A) = 0 ~\text{if}~ j \not\in A
\]
where the weight system $\lambda : D \times \pset{D} \rightarrow \R$ parameterizes the family. These allocations can be interpreted as redistributing the dividends back to the players. These allocations are efficient (\Propref{prop:effHarsa}) and generalize many classes of allocations (\Appendixref{adpx:allocs}). In this context, the well-known Shapley value arises as the egalitarian weight system $\lambda_j(A)=1/|A|$, yielding
\begin{equation}
    \Shap_v(j) \eqdef \sum_{A \in \pset{D} : j \in A} \frac{\varphi_v(A)}{|A|},
    \label{eq:shap}
\end{equation}
which divides each dividend equally among coalition members (\Appendixref{adpx:allocs:shap}).

Using $\lambda^{\text{PS}}_j(A) = v\!\lrpar{\lrcubra{j}}\bigl/\!\sum_{i \in A} v\!\lrpar{\lrcubra{i}}$ characterizes the \emph{proportional Shapley values} \cite{Shapley1953, Beal2018},
\begin{equation}
\PShap_v(j) \eqdef \sum_{A \in \pset{D}: j \in A} 
\frac{\absval{v\!\lrpar{\lrcubra{j}}}}{\sum_{j' \in A} \absval{v\!\lrpar{\lrcubra{j'}}}}\,
\varphi_v(A),
    \label{eq:pshap}
\end{equation}
which redistributes dividends in proportion to each of the players' individual values.

Drawing an analogy between players and model features, such allocations have become a cornerstone of feature-importance analysis. Efficiency makes them ideal for decomposing $v(D)$ into feature-level contributions. Specific choices of $v$ tailor the interpretation: model explanation by decomposing a prediction $\widehat{f}(x)$ \cite{Strumbelj2014, NIPS2017_8a20a862, Sundararajan2020}; importance quantification by decomposing the model's variance \cite{Owen2014, Fel2021, Herin2024}; or more advance uncertainty attributions by allocating kernel-based or information-theoretic uncertainty measures \cite{DaVeiga_kernelShap2021, Chau2022, watson2023explaining}.

%%================================%%
%% 3. UnACPI
\section{Uncertainty attribution of conformal prediction intervals}

%%================================%%
%% 3.1 Measures of uncertainty based on prediction intervals
\subsection{Measures of uncertainty based on prediction intervals}

We introduce a regression model-agnostic UA method based on uncertainty measures based on CP. More precisely, we introduce value functions derived from key quantities related to prediction intervals, which can be applied to the SMR, LACP, or CQR methods for constructing CP intervals. Each choice of value function defines a new cooperative game. We then propose to use the Shapley and proportional Shapley allocations of these various games to define novel UA influence measures. The present work focuses on three value functions based on CP intervals (CP interval width, and the two boundary points of the interval). Naturally, these value functions also depend on a data point, for which a prediction is performed. However, it is essential to note that the theoretical results presented in this section will still hold for any value function related to a prediction interval.

Let $\A$ be an algorithm that returns a regression model $\widehat{f}$ trained on $\TrSet$, and let $x$ be an instance of $X$. For every $A \in \pset{D}$, let $\widehat{f}_A$ be the output of $\A$ trained on $\TrSet_A$, and let $x_A$ be the subvector of $X$. Following Section~\ref{sec:scp}, for every $A \in \pset{D}$, let conformity scores $s_A : \calX_A \times \calY \times \calY^{\calX_A} \rightarrow \R$ and let $\calS_A \eqdef \lrcubra{s_A(X_i^{(A)}, Y_i, \widehat{f}_A) : (X_i^{(A)}, Y_i) \in \CalSet_A}$. Then, as in \eqref{eq:predInterval}, let $\widehat{C}_A\lrpar{x^{(A)}}$ be the CP interval related to the datapoint $x^{(A)}$.

\textbf{CP interval width}~ Our first proposed CP-based measure of uncertainty is the width of the CP interval, denoted \(\Delta\widehat{C}(x)\) as in \cite{MEHDIYEV2025110363, Yapicioglu2024}. To that end, we define the following value function:
\begin{equation*}
\forall A \in \pset{D}, ~\forall x^{(A)} \in \calX_A, \quad v_{\text{wCP}}(A, x) := \Delta\widehat{C}_A\lrpar{x^{(A)}}.
\end{equation*} 
Thus, this value function specializes to $v_{\text{wCP}}(A) = 2q_{1-\alpha}(\calS_A)$ for SMR, $v_{\text{wCP}}(A) = 2q_{1-\alpha}(\calS) \widehat{\sigma}(x^{(A)})$ for LACP and $v_{\text{wCP}}(A) = \widehat{q}_{1-\alpha/2}(x^{(A)}) - \widehat{q}_{\alpha/2}(x^{(A)}) + 2q_{1-\alpha}(\calS_A)$ for CQR. For SMR, the width-based feature importance remains the same across new observations.

\textbf{Boundary points}~ On top of the width of the interval, two key quantities would be its lower and upper bounds. For example, when assessing flood risks, one may be more interested in the value of the upper bound of the predicted river water level \cite{Lemaitre2015}. To that extent, we introduce the value functions 
\begin{equation*}
\forall A \in \pset{D}, ~\forall x^{(A)} \in \calX_A, \quad v_{\text{lowCP}}(A) := \inf \widehat{C}\lrpar{x^{(A)}} , \enspace v_{\text{upCP}}(A) := \sup \widehat{C}\lrpar{x^{(A)}}.
\end{equation*}
In the case of the SMR variant, we have $v_{\text{lowCP}}(A) = \widehat{f}(x^{(A)}) - q_{1-\alpha}(\calS_A)$, for LACP, $v_{\text{lowCP}}(A) = \widehat{f}(x^{(A)}) - q_{1-\alpha}(\calS_A)\widehat{\sigma}(x^{(A)})$; and for CQR, $v_{\text{lowCP}}(A) = \widehat{q}_{\alpha/2}(x^{(A)}) - q_{1-\alpha}(\calS_A)$.

\textbf{Normalization and allocations}~ We focus on the two allocations presented in \Secref{subsec:alloc}, the Shapley value \( \text{Shap}_{v_{\omega\text{CP}}}(j) \) and the proportional Shapley value \( \text{P\text{-}Shap}_{v_{\omega\text{CP}}}(j) \), where \( \omega \in \{\text{w}, \text{low}, \text{up} \} \). These two allocations offer desirable interpretative theoretical properties.
\begin{prop}
    For any data point $x \in \calX$, and value function $v_{\omega\text{\normalfont CP}}$, with $\omega \in \{\text{\normalfont w, low, up}\}$,
    $$\sum_{j=1}^d \text{\normalfont Shap}_{v_{\omega\text{\normalfont CP}}}(j,x) = v_{\omega\text{\normalfont CP}}(D,x) - v_{\omega\text{\normalfont CP}}(\emptyset) = \sum_{j=1}^d \text{\normalfont P-Shap}_{v_{\omega\text{\normalfont CP}}}(j,x).$$
    \label{prop:eff:shap:cp}
\end{prop}
Since both the Shapley and proportional Shapley values are efficient, considering two allocation schemes enables comparing their respective rankings. In practice, agreement between rankings produced by the two allocations enhances confidence in the results, as it signals a lack of sensitivity to the allocation choice. Meanwhile, disagreement can indicate that the situation is more complex and requires more attention. Moreover, the value functions can be normalized \wrt the baseline \( v_{\omega\text{CP}}(\emptyset) \), where \( \omega \in \{\text{w}, \text{low}, \text{up} \} \). Normalized value functions can be written, for $x \in \calX$, as 
\[
\forall A \in \pset{D}, ~\forall x^{(A)} \in \calX_A \quad \Tilde{v}_{\omega\text{CP}}(A, x^{(A)}) := \frac{v_{\omega\text{CP}}(A, x^{(A)}) - v_{\omega\text{CP}}(\emptyset)}{v_{\omega\text{CP}}(D, x) - v_{\text{CP}}(\emptyset)}.
\]
Hence, by leveraging \Propref{prop:eff:shap:cp}, normalizing the value functions implies that the resulting allocations will sum up to $1$. However, since there are no general theoretical guarantees on the monotonicity of the conformal scores of nested models, the allocations may fall outside the interval \([0, 1]\), but still sum to $1$. 

%%================================%%
%% 3.2 Estimation schemes
\subsection{Estimation and approximations}\label{sec:mc-approx}
This section presents an exact and a Monte Carlo-based approximation scheme to compute the Shapley and proportional Shapley values. The presented procedures are the same for any CP-based value function.

\textbf{Exact computations}~ The procedure to compute the exact Shapley and proportional Shapley values is fairly straightforward, and can be broken down into two steps. The first step requires evaluating the value functions. Thus, for every subset of features $A \in \pset{D}$, we must train the corresponding model $\widehat{f}_A$ using algorithm $\A$ on $\TrSet_A$, and then compute the conformity scores $\calS_A$ using $\CalSet_A$. As a whole, we need to train $2^d$ models and compute their conformity scores. Once we have access to the conformity scores, for a new data point $x \in \calX$, computing the prediction intervals $\widehat{C}_A(x^{(A)})$ for every $A \in \pset{D}$ allows extracting the evaluation of the value functions (e.g., width or boundary points of the interval). The second step entails aggregating the $2^d$ evaluated value functions to return either the Shapley or proportional Shapley values according to \eqref{eq:shap} or \eqref{eq:pshap}, respectively. This procedure is detailed in Algorithm~\ref{alg:UnACPI-1}.

\begin{algorithm}[ht!]
\caption{Exact computation procedure}\label{alg:UnACPI-1}
\begin{algorithmic}[1]
    \Require Data $\mathcal{D}=\{(X_i, Y_i)\}_{i=1}^{n}$, new data $\mathcal{D}^\text{New}=\{(X_i, Y_i)\}_{i=n+1}^{m}$, miscoverage level $\alpha\in(0,1)$, regression algorithm $f$, conformity score algorithm $s$, and a weight assignment $\lambda: \mathcal{P}_D \to \mathbb{R}$ that associates a weight $\lambda(A)$ to each subset $A \subseteq D$, where $D = \{1, \dots, d\}$ is the set of variable indices in $X$.
    %\Ensure Output description
    \State Randomly split $\mathcal{D}$ into two disjoint datasets $\mathcal{D}^{\text{Tr}}$ and $\mathcal{D}^{\text{Cal}}$
    \For{$A \in \mathcal{P}_D$}
        \State Define $\mathcal{D}^{\text{Tr}}_A = \{(X_{i,A}, Y_i)\}_{(X_i, Y_i) \in \mathcal{D}^{\text{Tr}}}$
        \State Define $\mathcal{D}^{\text{Cal}}_A = \{(X_{i,A}, Y_i)\}_{(X_i, Y_i) \in \mathcal{D}^{\text{Cal}}}$
        \State Define $\mathcal{D}^{\text{New}}_A = \{X_{i,A}\}_{X_i \in \mathcal{D}^{\text{New}}}$
        \State Train $\hat{f}_A$ on $\mathcal{D}^{\text{Tr}}_A$
        %\If Add Train \hat{\sigma}
        %\EndIf
        \State Compute conformity scores $\hat{s}_i$ for each $(X_i, Y_i) \in \mathcal{D}^{\text{Cal}}_A$
        \For{$X_i \in \mathcal{D}^{\text{New}}_A$}
            \State Compute conformal prediction interval $C_A(X_i)$
            \State Compute associated value $v(A; X_i)$
        \EndFor
    \EndFor
    % After this, we have $d$ values for each new observation
    % Then, we need to distribute the total gain to each feature
    \State Initialize matrix $\Phi \in \mathbb{R}^{m \times d}$ with zeros
    \For{$X_i \in \mathcal{D}^{\text{New}}$}
        \For{$j \in D$}
            \State Compute $\phi_i(j)$
            \State Store $\phi_i(j)$ in $\Phi_{i,j}$
        \EndFor
    \EndFor
    \State \Return $\Phi$
\end{algorithmic}
\end{algorithm}

\textbf{Approximation procedure}~ Overall, the bulk of the computational strain lies in training the models. A solution to drive the cost down would be to reduce the number of models to train from $2^d$ to a more controllable quantity. To address this, several strategies have been proposed in the literature. \cite{Rabitz1999,Li2001} proposed removing coalitions whose cardinality is above a certain threshold based on heuristics, \cite{Jethani2022} proposed to randomly sample according to a distribution over the coalitions, or \cite{Strumbelj2014, Song2016} proposed a Monte Carlo-type sampling of the permutations of players for the Shapley values. We generalize the latter approach to a broad class of allocations to produce an approximation scheme with statistical guarantees (see Theorems~\ref{thm:cvgmontecarlo} and \ref{thm:cvgimpsampling}).

The Shapley and proportional Shapley values are part of a class of allocations known as the Weber set \cite{Weber_1988} (see Appendix~\ref{adpx:allocs:shap}), which relies on the ordering of players. By randomly sampling $m$ of the $d!$ possible orderings, we require, at worst, $m \times d$ models to train, bypassing the exact computations' exponential scaling. The way the sampling is done dictates the allocation being approximated. For instance, a uniform sample of the permutations approximates the Shapley values. For the proportional values, a different distribution is needed (see Appendix~\ref{app:subsec:pshap:approx}). The complete approximation procedure is described in Algorithm~\ref{alg:WeberPermut}, with extensive details and explanations in Appendix~\ref{adpx:allocs}.
\begin{thm}[Statistical properties of Algorithm~\ref{alg:WeberPermut}]
\label{thm:cvgmontecarlo}
For any value function $v$, and any $j \in D$ the approximations \( \widehat{\text{Shap}_v}(j) \) and \( \widehat{\text{P-Shap}_v}(j) \) using Algorithm~\ref{alg:WeberPermut} are unbiased, strongly consistent, and asymptotically normal estimators of $\Shap_v(j)$ and $\text{P-Shap}_v(j)$, respectively.
\end{thm}
Moreover, it is possible to reweigh already sampled permutations to produce estimates for both the Shapley and proportional Shapley, without having to run Algorithm~\ref{alg:WeberPermut} again. This procedure relies on an importance sampling (IS) scheme detailed in Appendix~\ref{app:subsec:impsampling}.
\begin{thm}[Statistical properties of the IS approximation]
\label{thm:cvgimpsampling}
For any value function $v$, and any $j \in D$ the approximations \( \widehat{\text{Shap}}_{v,\textrm{IS}}(j) \) and \( \widehat{\text{P-Shap}}_{v,\textrm{IS}}(j) \) using the importance sampling scheme described in Appendix~\ref{app:subsec:impsampling} are unbiased, strongly consistent, and asymptotically normal estimators of $\Shap_v(j)$ and $\text{P-Shap}_v(j)$, respectively.
\end{thm}

%%===========================================================%%
%% 4. Experiments
\section{Experiments}\label{sec:experiments}
Further information on the datasets, data-generating processes, and supplemental experiments appears in \Appendixref{sec:addExps}. Reproducible codes, datasets, experiments, and figures are available in the accompanying GitHub repository\footnote{The contents of the GitHub repository are made available for review as a supplementary zip file.}.
% \footnote{\href{https://github.com/milidris/UnACPI}{https://github.com/milidris/UnACPI}}

%%===========================================================%%
%% 4.1 Sobol'-Levitan
\subsection{Empirical convergence of the approximations and runtime gains}\label{sec:sobol-levitan}

Beyond the theoretical guarantees in \Secref{sec:mc-approx}, we examine how the Monte Carlo approximation behaves as the number of sampled permutations increases. We adopt a modified Sobol'–Levitan benchmark \cite{SobolLevitan99, Moon2012, simulationlib} where $X = (X_1, \dots, X_{16})^\top \sim \calU(0,1)^{\times16}$ is comprised of mutually independent random variables, with target
$$Y = \exp \lrbra{ \beta^\top X} + \prod_{i=1}^{16}\frac{\exp\lrbra{\beta_i} - 1}{ \beta_i} + \epsilon_Y, \enspace \text{where }~\beta \in \R^{16}, \text{ and } \epsilon_Y \sim \calN(0,1).$$
A sampled dataset of size $1,000$ is randomly split $80\% - 20\%$ for training and calibration, respectively. We decompose the CP interval width using the SMR method on linear models on $50$ test samples using the Shapley values as a baseline. For permutation counts $m\in\{100,200,\dots,5,000\}$ we perform $150$ repetitions of the Monte Carlo approximations, recording the estimated Shapley values, the number of fitted linear models, and the effective runtime\footnote{Single-core runs on an AMD Ryzen PRO 4750U with 32 GB RAM, \texttt{R} 4.2.1.}. \Figref{fig:mc-conv} (and the full curves in \Secref{apdx:results-sobolativan}) show convergence toward the exact values (Panel A), a model count far below the worst-case $m\times d$ (Panel B), and, for example, a ten-fold speed-up at $m=1,000$ (Panel C).
\begin{figure}[t!]
    \centering
    \includegraphics[width=\linewidth]{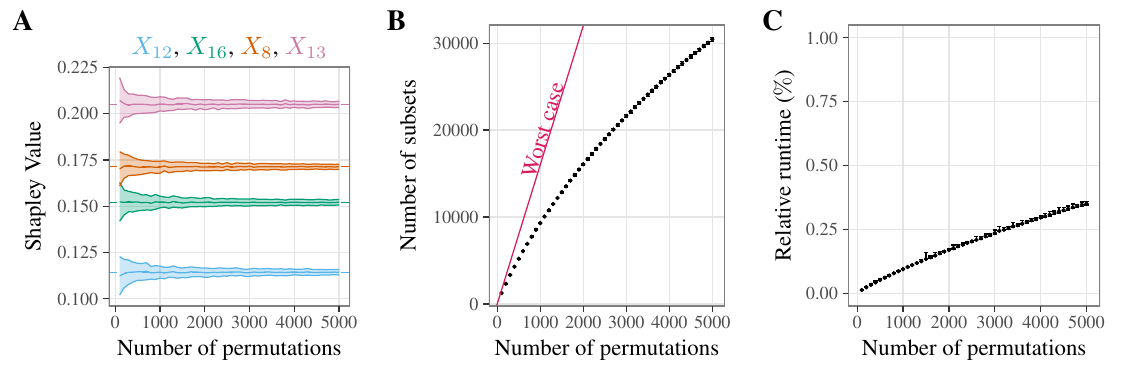}%
    \caption{Empirical convergence for the four most important features (\textbf{A}), number of trained models and worst case scenario $m \times d$ (\textbf{B}), and runtime relative to exact computations (\textbf{C}) for the empirical study of the Monte Carlo approximation scheme}\label{fig:mc-conv}
\end{figure}
%%===========================================================%%
%% 4.2 Friedman
\subsection{Feature selection beyond moment-dependent importance}\label{sec:friedman}
We next compare importance rankings based on CP intervals with traditional rankings derived from the conditional mean (SHAP \cite{NIPS2017_8a20a862}) and conditional variance. Using a variant of Friedman's benchmark \cite{Friedman1991}, as in \cite{watson2023explaining}, let $X=(X_1,\dots,X_{11})^\top\sim\calU(0,1)^{\times11}$ and define  
\[
\begin{aligned}
V &= 10\sin(\pi X_1X_2)+20(X_3-0.5)^2+10X_4+5X_5+\epsilon_V,\quad\epsilon_V\sim\calN(0,1),\\
Z &= 10\sin(\pi X_6X_7)+20(X_8-0.5)^2+10X_9+5X_{10}+\epsilon_Z,\quad\epsilon_Z\sim\calN(0,1),
\end{aligned}
\]
and the target variable is distributed according to $Y \sim \calN(Z, V)$. Thus $(X_6,\dots,X_{10})$ govern the mean, $(X_1,\dots,X_5)$ the heteroskedasticity, and $X_{11}$ is noise.

We generate a training set of size $10,000$, and calibration and test sets of size $5,000$ each. We fit LightGBM models, with 30 trees each with a maximum depth of 10 leaves. On the training data, we fit models of the conditional mean, the conditional variance (using the squared residuals), and along with the calibration data, we compute the LACP and CQR (using the pinball loss with levels equal to $0.1$ and $0.9$) intervals with coverage of $99\%$ ($\alpha = 0.01$). We then compute the Shapley values decompositions on the test set, resulting in \Figref{fig:res-friedman}. We can notice that both LACP and CQR uncertainty measures based on width are between the conditional mean and the conditional variance. This demonstrates that these quantile-based measures do go beyond typical moment-based importance measures to quantify importance, effectively adding a different interpretative layer to the XAI practitioner's toolbox.
\begin{figure}[ht!]
    \centering
    \includegraphics[width=\linewidth]{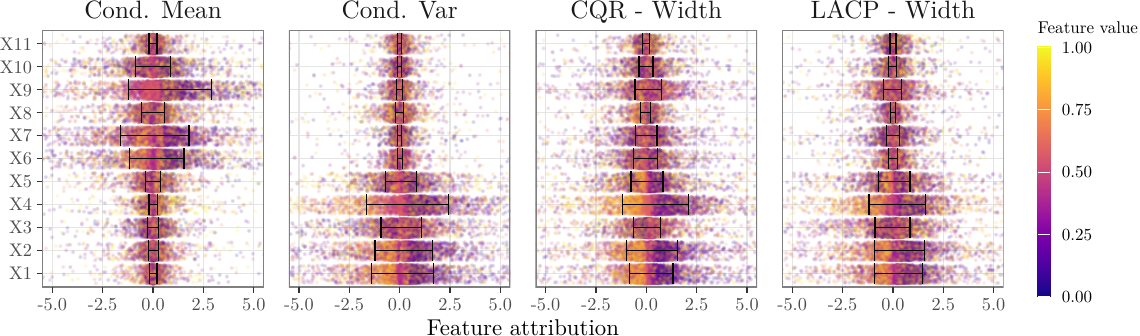}
    \caption{Feature attribution for the modified Friedman example. The bars mark the 90\% intervals.}
    \label{fig:res-friedman}
\end{figure}

%%===========================================================%%
%% 4.3 Real-world datasets
\subsection{Real-world datasets}
We conduct experiments on real-world classical datasets for CP \cite{Romano2019,xie2024boosted}, detailed in Table~\ref{tab:datasets}. They are chosen for their diversity in size, dimensionality, and data types. $20\%$ of each dataset forms a held-out test set; the remainder is split $80\!:\!20$ into training and calibration sets, respectively. We use either exact computations or Monte Carlo approximations depending on the dataset.

\begin{table}[ht!]
    \centering
    \begin{tabular}{l|C{2.3in}|c|c|c|r}
        \hline
        \textbf{Dataset} & \textbf{Description} & $n$ & $d$ & Estimation &\textbf{Source} \\
        \hline
        \texttt{bike}       & Bike rental data                              & 17,379    & 12    & Exact  &\cite{BikeSharingData}    \\
        \texttt{blog}       & Number of comments per blog posts             & 52,397    & 238   & $m=50$ &\cite{BlogData}           \\
        \texttt{casp}       & Physicochemical properties of proteins        & 45,730    & 9     & Exact  &\cite{PhysicoData}        \\
        \texttt{concrete}   & Concrete compressive strength                 & 1,030     & 8     & Exact  &\cite{concreteData}       \\
        \texttt{facebook}   & Engagement of facebook posts                  & 79,788    & 37    & $m=200$&\cite{facebookData}       \\
        \texttt{UScrime}    & Crime data in the US                          & 1,993     & 101   & $m=200$&\cite{crimeData}          \\
        \texttt{star}       & Effect of reducing class size on test scores  & 2,161     & 38    & $m=200$&\cite{starData}          \\
        \hline
    \end{tabular}
    \caption{Real world dataset line-up for realistic experiments}
    \label{tab:datasets}
\end{table}
\textbf{Models and hyperparameters}~ For SMR and LACP we fit linear regression (LR), LightGBM (LGB), and random-forest (RF) models; for CQR we use quantile LR (Q-LR) and quantile RF (Q-RF).\footnote{\texttt{R} Packages: \href{https://search.r-project.org/R/refmans/stats/html/00Index.html}{\texttt{stats}}, \href{https://CRAN.R-project.org/package=lightgbm}{\texttt{lightgbm}}, \href{https://CRAN.R-project.org/package=randomForest}{\texttt{randomForest}}, \href{https://CRAN.R-project.org/package=quantreg}{\texttt{quantreg}}, and \href{https://CRAN.R-project.org/package=quantregForest}{\texttt{quantregForest}}, respectively.} All $\widehat{\sigma}$ models in LACP use LGB. Unless specified, package defaults are retained. For RF and Q-RF we set the minimum node size to $20\%$ of training cases, limit the number of trees to $75$, and choose $\texttt{mtry} =\sqrt{d}$. LGB uses $100$ boosting rounds, except for \texttt{UScrime}, \texttt{star}, and \texttt{blog}, where $25$ rounds are adopted. Q-LR employs the Frisch–Newton algorithm with default machine precision.

\textbf{Selected results}~ We present a few selected interesting results. We refer the interested reader to Appendix~\ref{sec:addExps} for more details on the remainder of the visualizations. First, we showcase how our approach differs from standard approaches regarding importance rankings. \Figref{fig:concrete-ellipses} contrasts importance rankings based on interval width (LACP) with those based on the conditional mean (SHAP) for RF and LGB on the \texttt{concrete} data. The center of the ellipses represents the mean ranks over the test set. Uncertainties, are represented in the shape of the ellipses. Vertical and horizontal elongations are proportional to the ranking standard deviation for the conditional mean and the CP interval width, respectively. One can notice that our proposed approach does not yield the same importance rankings as the conditional mean-based feature attribution. This is explained by the CP framework's reliance on quantiles, which are much more general than moment-based statistics, as already demonstrated in \Secref{sec:friedman}.
\begin{figure}[ht!]
    \centering
    \includegraphics[width=\linewidth]{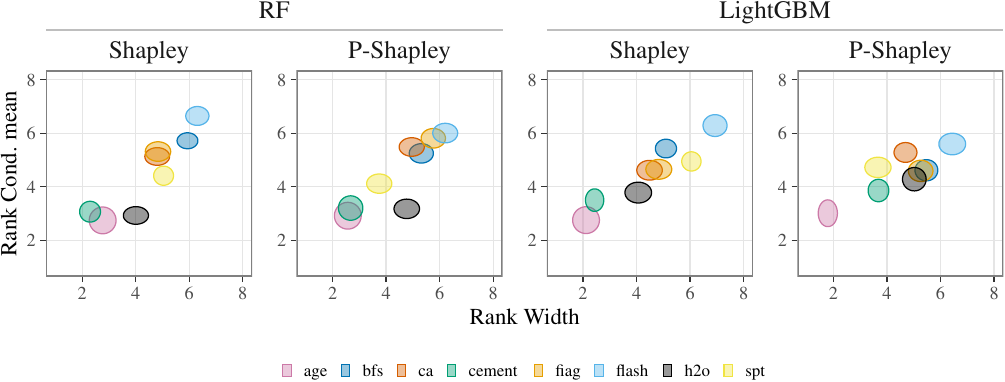}
    \caption{LACP width-based vs. conditional mean-based importance rankings for RF and LGB models on the \texttt{concrete} dataset.}
    \label{fig:concrete-ellipses}
\end{figure}
\Figref{fig:facebook-topranks} depicts rank frequencies for (absolute value of) the Shapley and proportional Shapley allocations of the upper bound (CQR) on the \texttt{facebook} dataset. Panel~(\textbf{A}) shows the complete rank distributions, while Panel~(\textbf{B}) highlights the most frequent top-ranked features. One can notice that, on the \texttt{facebook} dataset, the allocation choice can impact the features' overall ranking. Unchanged rankings can indicate stability (and thus confidence) between the allocations. In contrast, the multiplicities of differing rankings enable a more complete depiction of the importance, since a choice of allocation only tells one side of the story. Moreover, the overall feature ranking also depend on the model choice. Not only the rankings differ between Q-RF and Q-LR models, but also the frequencies at which these rankings are observed. This indicates a more nuanced and locally diverse definition of uncertainties.
\begin{figure}[ht!]
    \centering
    \includegraphics[width=\linewidth]{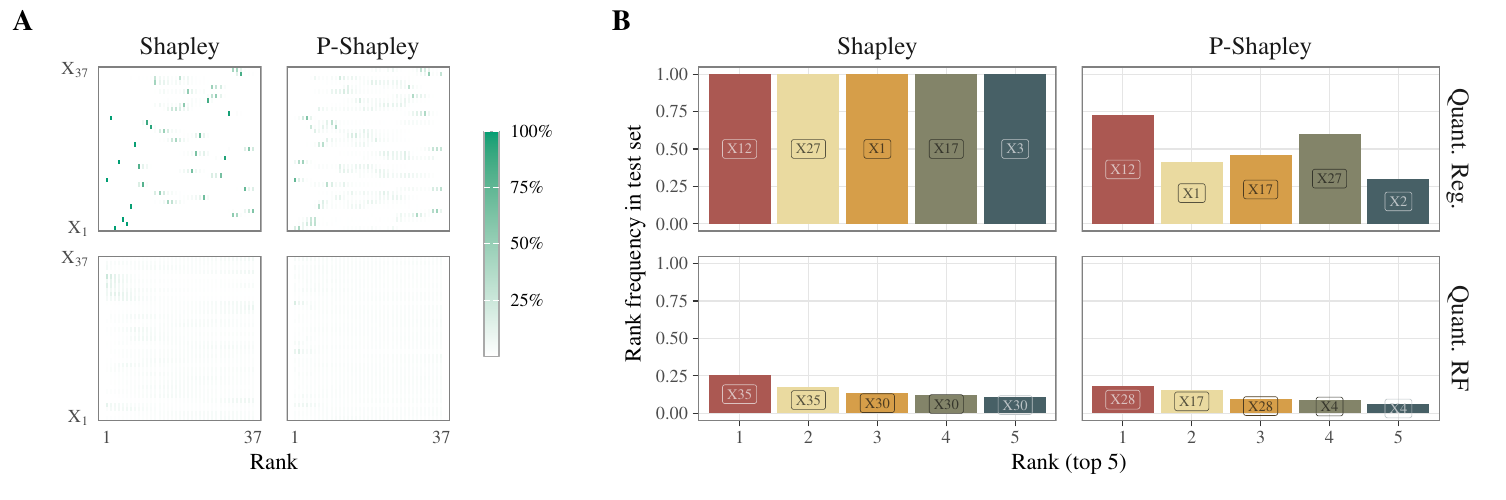}
    \caption{Rank frequency of all the features (\textbf{A}) and top 5 most important feature (\textbf{B}) over the test data. CQR upper bound-based importance rankings for Q-LR and Q-RF models on the \texttt{facebook} dataset are used.}
    \label{fig:facebook-topranks}
\end{figure}
%%===========================================================%%
%% 5. Conclusion
\section{Conclusion}
We proposed a regression-model-agnostic uncertainty attribution (UA) framework that combines cooperative-game theory with conformal prediction (CP) to interpret predictive uncertainty feature-wise. By defining cooperative games whose value functions encode CP-interval properties, we extend attribution beyond mean predictions and offer insights suited to high-stakes settings. Building on the Shapley value, we employed the broader Harsanyi allocation family and highlighted proportional Shapley values, which distribute uncertainty proportionately to individual feature contributions. We introduced a Monte Carlo approximation with proven unbiasedness, consistency, and asymptotic normality to curb the high computational cost, making large-scale applications feasible. Experiments on synthetic benchmarks and varied real-world datasets showed that CP-based uncertainty attributions diverge markedly from moment-based rankings, illustrating our method's added value. Future work includes automatically defining optimal allocations specialized for XAI tasks, exploring other venues to link XAI and the CP framework (e.g., beyond SCP, CP intervals for time series or sets) \cite{Zaffran2024} and the statistical study and adaptation of the approximation schemes based on surrogate models \cite{NIPS2017_8a20a862, Jethani2022}.

%%===========================================================%%
%% References
\bibliography{biblio}

%%===========================================================%%
%% Appendices
%%===========================================================%%
\newpage
\appendix

%%===========================================================%%
%% Allocations
\section{Sets of Allocations and Monte Carlo Approximations}
\label{adpx:allocs}

In this appendix, cooperative games are understood as a tuple $(D,v)$ where $D$ is a finite set of players, and $v : \pset{D} \rightarrow \R$ is a set function called the value function of the cooperative game. An allocation is a set function $\phi_v : D \rightarrow \R$ related to a cooperative game $(D,v)$. 

\subsection{Equivalent Formulations of the Shapley Values and Corresponding Sets of Allocations}
\label{adpx:allocs:shap}
The Shapley values $\Shap_v : \pset{D} \rightarrow \R$ of a cooperative game $(D,v)$ are a very well-known allocation. It has become popular as the only allocation that respects the four axioms:
\begin{itemize}
    \item Efficiency: $\sum_{i\in D} \Shap_v(i) = v(D) - v(\emptyset)$;
    \item Symmetry: for $i,j \in D$, if for every $A \in \pset{D}$, $v(A \cup \{i\}) = v(A \cup \{i\})$, then $\Shap_v(i) = \Shap_v(j)$;
    \item Sensitivity (or Dummy/Null player): for $i \in D$, if for every $A \in \pset{D}$, $v(A \cup \{i\}) = v(A)$, then $\Shap_v(i) = 0$;
    \item Linearity (or Additivity): for two cooperative games $(D,v)$ and $(D,v')$, $\Shap_{v+v'} = \Shap_v + \Shap_{v'}$.
\end{itemize}
Shapley found that the resulting allocation admits an analytical formula. The original formulation writes \cite{Shapley1951}
\begin{equation}
    \forall j \in D, \enspace \Shap_v(j) = \sum_{A \in \mathcal{P}_D ~ : ~ j \notin A} \frac{|A|!(d-|A|-1)!}{d!} \left[v(A \cup \{j\}) - v(A)\right].
    \label{eq:shapOriginal}
\end{equation}

\begin{rmk}[Shapley values and model explanations]
    The Shapley values have been extensively used in the XAI literature for ``model explanations'' \cite{Strumbelj2014, NIPS2017_8a20a862}. For a model $\widehat{f} : \calX \rightarrow \calY$ %learned from an algorithm $\A$, 
    and an observation $x \in \calX$, it amounts to choosing a value function such that $v(D) = \widehat{f}(x)$ (see e.g., \cite{Sundararajan2020} for an overview). It has become particularly popular since, for such value functions, thanks to the efficiency property, it allows writing
    \begin{equation*}
        \widehat{f}(x) - \mathbb{E}\left[\widehat{f}(x)\right] = \sum_{j \in D} \Shap_v(j),
    \end{equation*}
    where the summands are interpreted as parts of the prediction attributed to each feature of the learned model.
\end{rmk}
With advances in the theoretical study of cooperative games, different sets of efficient allocations have been introduced. For the purposes of this article, two of them are introduced: the Weber and Harsanyi sets. They are particularly interesting since it has been shown that the Shapley values are part of these sets of allocations, resulting in formulas equivalent to \Eqref{eq:shapOriginal}, allowing for finer interpretations of the allocations, going beyond the axiomatic characterizations.

\paragraph{Shapley Values as Uniform Choice Over Random Orders: the Weber Set} The Weber set of allocations of a cooperative game $(D,v)$ can be understood from a random order perspective. Let $\calS_D$ be the set of permutations of $D$ (without replacement). For an ordering $\pi = (\pi_1, \dots, \pi_d) \in \calS_D$ and for a any player $j\in D$, let $\pi(j)$ be the position of player $j$ in the ordering $\pi$ (i.e., $\pi_{\pi(j)}=j$), and let $\pi^j$ be the set of players preceding $j$ in $\pi$, including $j$ (i.e., $\pi^j \eqdef \lrcubra{i \in D: \pi(j) \leq \pi(i)}$). Let $p : \calS_D \rightarrow [0,1]$ be a probability mass function over $\calS_D$, called the random order distribution. The Weber set \cite{Weber_1988} is the set of allocations, parametrized by $p$, that can be written as
\begin{equation}
    \forall j \in D, \enspace \phi_v(j) = \sum_{\pi \in \calS_D} p(\pi) \lrbra{v\lrpar{\pi^j} - v\lrpar{\pi^j \setminus \lrcubra{j}}} = \mathbb{E}_{p} \lrbra{v\lrpar{\pi^j} - v\lrpar{\pi^j \setminus \lrcubra{j}}},
    \label{eq:webAllocations}
\end{equation}
which can be interpreted as the expectation \wrt $p$ of the marginal contributions of each player in all the possible orderings.
\begin{prop}[Efficiency of the Weber set]
    Allocations of the form of \Eqref{eq:webAllocations} are efficient.
    \label{prop:effWeber}
\end{prop}
Even though Proposition~\ref{prop:effWeber} is a very well-known fact in the theory of cooperative games \cite{Weber_1988}, we provide a proof for completeness' sake.
\begin{proof}[Proof of Proposition~\ref{prop:effWeber}]
    First, we notice that, for any $\pi \in \calS_D$, $\sum_{j \in \pset{D}} \lrbra{v\lrpar{\pi^j} - v\lrpar{\pi^j \setminus \lrcubra{j}}}$ is a telescoping series, which is equal to
    \begin{equation*}
        v(\lrcubra{\pi_1}) - v(\emptyset) + v(\lrcubra{\pi_1, \pi_2}) - v(\lrcubra{\pi_1}) +\dots+ v(\lrcubra{\pi_1, \dots, \pi_d}) - v(\lrcubra{\pi_1,\dots,\pi_{d-1}})= v(D) - v(\emptyset).
    \end{equation*}
    Thus,
    \begin{align*}
        \sum_{j \in \pset{D}} \sum_{\pi \in \calS_D} p(\pi) \lrbra{v\lrpar{\pi^j} - v\lrpar{\pi^j \setminus \lrcubra{j}}} &= \sum_{\pi \in \calS_D} p(\pi) \lrpar{\sum_{j \in \pset{D}} \lrbra{v\lrpar{\pi^j} - v\lrpar{\pi^j \setminus \lrcubra{j}}}}\\
        &=\lrpar{v(D) - v(\emptyset)}\sum_{\pi \in \calS_D} p(\pi) = v(D) - v(\emptyset)
    \end{align*}
    where the last equality comes from the fact that $p$ is a probability mass function over $\calS_D$.
\end{proof}

Different choices of random order distributions lead to different allocations. In particular, choosing $p$ as a discrete uniform distribution over $\calS_D$ (i.e., $p(\pi) = 1/d!, ~ \forall \pi \in \calS_D$), denoted $\calU(\calS_D)$, coincides with the Shapley values of the game $(D,v)$, leading to the different, but equivalent, formulation 
\begin{equation}
    \forall j \in D, \enspace \Shap_v(j) = \frac{1}{d!} \sum_{\pi \in \calS_D} \lrbra{v\lrpar{\pi^j} - v\lrpar{\pi^j \setminus \lrcubra{j}}} = \mathbb{E}_{\calU(\calS_D)} \lrbra{v\lrpar{\pi^j} - v\lrpar{\pi^j \setminus \lrcubra{j}}}.
    \label{eq:shapWeber}
\end{equation}

\begin{exa}[Random order formulation for 3 players]
We derive the computation of the Shapley values according to \Eqref{eq:shapWeber} in the case where $D=\lrcubra{1,2,3}$. In this case, the $6$ permutations are given by
\begin{equation*}
    \calS_D = \lrcubra{ (1,2,3), (1,3,2), (2,3,1), (3,2,1), (2,1,3), (3,1,2)}.
\end{equation*}
For every permutation, the marginal contributions (MC) for each player are displayed in Table~\ref{tab:shapWeberMC}.
\begin{table}[ht!]
    \centering
    \begin{tabular}{c|c|c|c}
         Ordering   & Marg. Cont. player $\{1\}$  & Marg. Cont. player $\{2\}$  & Marg. Cont. player $\{3\}$    \\
         \hline
         $(1,2,3)$  & $v(\{1\}) - v(\emptyset)$   & $v(\{1,2\}) - v(\{1\})$     & $v(\{1,2,3\}) - v(\{1,2\})$   \\
         $(1,3,2)$  & $v(\{1\}) - v(\emptyset)$   & $v(\{1,2,3\}) - v(\{1,3\})$ & $v(\{1,3\}) - v(\{1\})$       \\
         $(2,3,1)$  & $v(\{1,2,3\}) - v(\{2,3\})$ & $v(\{2\}) - v(\emptyset)$   & $v(\{2,3\}) - v(\{2\})$       \\
         $(3,2,1)$  & $v(\{1,2,3\}) - v(\{2,3\})$ & $v(\{2,3\}) - v(\{3\})$     & $v(\{3\}) - v(\emptyset)$     \\
         $(2,1,3)$  & $v(\{1,2\}) - v(\{2\})$     & $v(\{2\}) - v(\emptyset)$   & $v(\{1,2,3\}) - v(\{1,2\})$   \\
         $(3,1,2)$  & $v(\{1,3\}) - v(\{3\})$     & $v(\{1,2,3\}) - v(\{2,3\})$ & $v(\{3\}) - v(\emptyset)$     \\
         \hline
    \end{tabular}
    \caption{Marginal contributions of each player for every possible ordering of $3$ players.}
    \label{tab:shapWeberMC}
\end{table}
The average of the marginal contributions of, e.g., player $1$ with uniform weights (equal to $1/6$ here) leads to
\begin{align*}
    \Shap_v(1) &= 2 \times \frac{v(\{1\}) - v(\emptyset)}{6} + 2\times \frac{v(\{1,2,3\}) - v(\{2,3\})}{6}\\
    &+ \frac{v(\{1,2\}) - v(\{1\})}{6} + \frac{v(\{1,3\}) - v(\{3\})}{6}
\end{align*}
which ends up being equivalent to the original formulation above. In general, with three players, one has that, by letting $D=\{i,j,k\}$,
\begin{equation*}
        \Shap_v(j) = 2\times \frac{v(\{j\}) - v(\emptyset) + v(D) - v(D\setminus \{j\})}{6} + \frac{v(\{j,k\}) - v(\{k\}) + v(\{j,i\}) -v(\{i\})}{6},
\end{equation*}
which remains equivalent to the original formula for the Shapley values.
\end{exa}

\paragraph{Shapley Values as an Egalitarian Redistribution of Dividends: the Harsanyi Set}
The allocations in the Harsanyi set, as presented in \Secref{subsec:alloc}, are efficient for any choice of valid weight system.
\begin{prop}[Efficiency of the Harsanyi set]
    Allocations in the Harsanyi set are efficient.
    \label{prop:effHarsa}
\end{prop}
The efficiency of allocations in the Harsanyi set has been cataloged in the theory of cooperative games (see e.g., \cite{Vasilev2001}). However, for completeness' sake, we provide a proof of this fact.
\begin{proof}[Proof of Proposition~\ref{prop:effHarsa}]
    Recall that the Harsanyi dividends are defined as
    $$\forall A \in \pset{D}, \enspace \varphi_v(A) = \sum_{B \in \pset{A}} (-1)^{|A|-|B|}v(B),$$
    and from Rota's generalization of the Möbius inversion formula on powersets \cite{Rota1964}, it implies that:
    $$\forall A \in \pset{D}, \enspace v(A) = \sum_{B \in \pset{A}} \varphi_v(A).$$
    Notice that, by definition $\varphi_v(\emptyset) = v(\emptyset)$, and thus, in particular, for $A = D$ in the previous equation, we have that 
    $$\sum_{A \in \pset{D}} \varphi_v(A) = v(D) \iff \sum_{A \in \pset{D} : A \neq \emptyset} \varphi_v(A) = v(D) - v(\emptyset).$$
    Next, notice that:
    $$\sum_{A \in \pset{D}~ :~ j \in A} \lambda_j(A) \varphi_v(A) = \sum_{A \in \pset{D}~:~A \neq \emptyset} \lambda_j(A) \mathds{1}_{\lrcubra{j \in A}} \varphi_v(A).$$
    Thus,
    \begin{align*}
        \sum_{j \in D} \sum_{A \in \pset{D}~ :~ j \in A} \lambda_j(A) \varphi_v(A) &= \sum_{j \in D}\sum_{A \in \pset{D}~:~A \neq \emptyset} \lambda_j(A) \mathds{1}_{\lrcubra{j \in A}} \varphi_v(A)\\
        &=\sum_{A \in \pset{D}~:~A \neq \emptyset}\varphi_v(A)\lrpar{\sum_{j \in D}\lambda_j(A) \mathds{1}_{\lrcubra{j \in A}}}\\
        &= \sum_{A \in \pset{D}~:~A \neq \emptyset}\varphi_v(A)\lrpar{\sum_{j \in D}\lambda_j(A)} = \sum_{A \in \pset{D}~:~A \neq \emptyset}\varphi_v(A),
    \end{align*}
    since the weight system must respect $\sum_{j \in D} \lambda_j(A) = 1$. Thus,
    $$\sum_{j \in D} \sum_{A \in \pset{D}~ :~ j \in A} \lambda_j(A) \varphi_v(A) = v(D) - v(\emptyset),$$
    and efficiency follows.
\end{proof}

Different choices of weight systems lead to different allocations. In particular, choosing the egalitarian redistribution of the dividends \cite{Harsanyi1963} (i.e., $\forall j \in D, ~\forall A \in \pset{D}, \enspace \lambda_j(A) = 1/|A|$) coincides with the Shapley values of $(D,v)$, resulting in the different, but equivalent formulation of the Shapley values
\begin{equation}
    \forall j \in D, \enspace \Shap_v(j) = \sum_{A \in \pset{D} ~ : ~ j \in A} \frac{\varphi_v(A)}{|A|}.
    \label{eq:shapHarsa}
\end{equation}
The egalitarian nature of the redistribution of the dividends can be understood as follows: for any $A \in \pset{D}$, $\varphi(A)$ is divided in $|A|$ equal parts, which are redistributed to each of the players in $A$.

\begin{exa}[Dividend-sharing formulation for 3 players]
We derive the computation of the Shapley values according to \Eqref{eq:shapHarsa} in the case where $D=\{1,2,3\}$. The dividends are given in Table~\ref{eq:shapHarsa}.
\begin{table}[ht!]
    \centering
    \begin{tabular}{c|c|c|c|c}
         Coalition      & Harsanyi dividend                                                 & $\lambda_1$ & $\lambda_2$ & $\lambda_3$   \\
         \hline
         $\emptyset$    & $v_{\emptyset}$                                                   & 0           & 0           & 0             \\
         $\{1\}$        & $v_1- v_{\emptyset}$                                              & 1           & 0           & 0             \\
         $\{2\}$        & $v_2- v_{\emptyset}$                                              & 0           & 1           & 0             \\
         $\{3\}$        & $v_3- v_{\emptyset}$                                              & 0           & 0           & 1             \\
         $\{1,2\}$      & $v_{12} - v_1 - v_2 +v_{\emptyset}$                               & 1/2         & 1/2         & 0             \\
         $\{1,3\}$      & $v_{13} - v_1 - v_3+v_{\emptyset}$                                & 1/2         & 0           & 1/2           \\
         $\{2,3\}$      & $v_{23} - v_2 - v_3+v_{\emptyset}$                                & 0           & 1/2         & 1/2           \\
         $\{1,2,3\}$    & $v_{123} - v_{12}- v_{13}- v_{23} +v_1 +v_2+v_3- v_{\emptyset}$   & 1/3         & 1/3         & 1/3           \\
         \hline
    \end{tabular}
    \caption{Harsanyi dividends and weight system for the Shapley values of a game with $3$ players.}
    \label{tab:shapHarsa}
\end{table}
Thus, for $j=1$, the reweighting of the dividends leads to:
\begin{align*}
    \Shap_v(j) &= v(\emptyset)\lrpar{-1 + \frac{1}{2} + \frac{1}{2} - \frac{1}{3}} +v(\{1\})\lrpar{1 -\frac{1}{2} - \frac{1}{2} + \frac{1}{3}} + v(\{2\})\lrpar{-\frac{1}{2} + \frac{1}{3}}\\
    & + v(\{3\})\lrpar{-\frac{1}{2} + \frac{1}{3}} + v(\{1,2\}) \lrpar{\frac{1}{2} - \frac{1}{3}} + v(\{1,3\}) \lrpar{\frac{1}{2} - \frac{1}{3}}\\
    &+ v(\{2,3\})\lrpar{-\frac{1}{3}} + \frac{v(\{1,2,3\})}{3}
\end{align*}
$$=\frac{v(\{1\}) - v(\emptyset)}{3} +\frac{v\lrpar{\{1,2,3\}}- v(\{2,3\})}{3} + \frac{ v(\{1,2\})- v(\{2\})}{6} + \frac{v(\{1,3\})  - v(\{3\})}{6}$$
and we recover the original formula for the Shapley values.
\end{exa}

\paragraph{General Relationship Between the Weber and Harsanyi Set}
The Weber set of allocations is a subset of the Harsanyi set. Thus, for a given random order distribution $p(\pi)$, the corresponding weight system is given by (see \cite{Derks2005})
\begin{equation}
    \forall j \in D,~ \forall A \in \pset{D},\quad \lambda_j(A) = \sum_{\pi \in \calS_D~:~ A \in \pi^i} p(\pi),
    \label{eq:lambdaPi}
\end{equation}
such that the resulting allocations are equivalent. Moreover, it is possible to find the corresponding random order distribution for certain weight systems. More precisely, \cite{Derks2002} (Theorem 5.5 and Corollary 5.7) showed that the Weber set coincides with the Harsanyi set restricted to weight systems satisfying:
$$\forall A \in \pset{D},~ \forall j \in A, \quad \sum_{B \in \pset{D}} (-1)^{|A|-|B|}\lambda_j(B) \geq 0,$$
and consequently, for such weight systems, it is also possible to write their corresponding random order distribution.

\subsection{Faster Computations: Sampling Random Orders}\label{app:subsec:weber:approx}

Exact estimators of allocations from either the Weber or the Harsanyi set require estimating the value function $v$ on all the $2^d -1$ coalitions of players. This quickly becomes intractable in practice for moderate to large numbers of features.
%A solution to this problem would be only estimating the value function on a selected subset of $\pset{D}$. 
%Several strategies have been proposed in the literature, such as removing coalitions whose cardinality is above a certain threshold \cite{Rabitz1999,Li2001}, randomly sampling according to a distribution over the coalitions for the Shapley values \cite{Jethani2022}, or Monte Carlo-type sampling the permutations of $D$ for the Shapley values \cite{Strumbelj2014, Song2016}. 
We propose a Monte Carlo-type sampling strategy over permutations of \( D \) to accommodate our methodological contributions, providing unbiased, consistent, and approximately normal estimates for any allocation in the Weber set (provided $p(\pi)$ is known).
First, let us recall the contributions of \cite{Strumbelj2014, Song2016} for the Shapley values. The authors leveraged \Eqref{eq:shapWeber} by noticing that it is an expectation over uniformly distributed permutations. Hence, by uniformly sampling a subset $\Pi_m\subseteq \calS_D$ of $|\Pi_m|=m \ll d!$ permutations, an approximation of this expectation is given by
$$ \forall j \in D, \enspace \widehat{\text{Shap}_v}(j) = \frac{1}{m} \sum_{\pi \in \Pi_m}  \lrbra{v\lrpar{\pi^j} - v\lrpar{\pi^j \setminus \lrcubra{j}}},$$
and unbiasedness, consistency, and asymptotic normality of this estimator follows from basic Monte Carlo sampling arguments.

Now, the same rationale can be applied for an arbitrary, but known, probability mass function $p(\pi)$. It requires randomly sampling $\Pi_m \subseteq \calS_D$ according to $p(\pi)$, and the approximation of the corresponding Weber allocation is given by:
\begin{equation}
    \forall j \in D, \enspace \widehat{\phi_v}(j) = \frac{1}{m} \sum_{\pi \in \Pi_m}  \lrbra{v\lrpar{\pi^j} - v\lrpar{\pi^j \setminus \lrcubra{j}}},
    \label{eq:approxWeberAlloc}
\end{equation}
yielding unbiased, consistent, and asymptotically normal estimators. This generalized Monte Carlo approximation scheme for allocations in the Weber set is illustrated in Algorithm~\ref{alg:WeberPermut}.

\begin{algorithm}[ht!]
\caption{Monte Carlo-type random order allocation approximation}\label{alg:WeberPermut}
\begin{algorithmic}[1]
    \Require Cooperative game $(D=\{1,\dots, d\},v)$
    \Require Number of permutations $m \ll d!$
    \Require Probability measure $p(\pi)$ over $\calS_D$
    \Ensure Approximation $\widehat{\phi_v}=\lrpar{\widehat{\phi_v}(1), \dots,\widehat{\phi_v}(d)}$ of Weber allocation $\phi_v$ with $p(\pi)$
    \For{$k \in \{1,\dots,m\}$}
        \State Sample $\pi_k \in \calS_D$ according to $p(\pi)$
        \For{$j =1, \dots, d$}
            \State Compute $MV_{\pi_k}^{j} = v\lrpar{\pi_k^j} - v\lrpar{\pi_k^j \setminus \lrcubra{j}}$
        \EndFor
    \EndFor
    \For{$j \in D$}
        \State Compute $\widehat{\phi_v}(j) = \frac{1}{m} \sum_{k=1}^m MV_{\pi_k}^j$
    \EndFor
    \State \Return $\widehat{\phi_v} =\lrpar{\widehat{\phi_v}(1), \dots, \widehat{\phi_v}(d)}$
\end{algorithmic}
\end{algorithm}

\subsection{Random Order Sampling for the Proportional Shapley Values} \label{app:subsec:pshap:approx}
The proportional Shapley values are part of a broader set of allocations, called the weighted Shapley values. They are characterized by weight systems having the form $\lambda^{\textrm{WS}}_i(A) = \frac{w(i)}{w(A)}$ for some arbitrary weight function $w : \pset{D} \rightarrow \R$. In \cite{Dehez2017}, the author linked this set of allocations to their random order expression. The random order distribution is then given by:
\begin{equation}
    \forall \pi =(\pi_1, \dots, \pi_d) \in \calS_D, \quad \textrm{WS}(\pi) \eqdef \exp \lrbra{ - \sum_{j=2}^d \log\lrpar{1+ \sum_{k=1}^{d-1} \frac{w\lrpar{\pi_j}}{w\lrpar{\pi_k}}}}.
    \label{eq:piWS}
\end{equation}
The proportional Shapley values in \Eqref{eq:piWS} can be seen as the particular case of weight function $w(A) = \sum_{j\in A}v(\{j\})$ for every $A \in \pset{D}$. Hence, the random order distribution for the proportional Shapley values is given by
\begin{equation}
    \forall \pi \in \calS_D, \quad  \textrm{PS}(\pi) \eqdef \exp \lrbra{ - \sum_{j=2}^d \log\lrpar{1 + \sum_{k=1}^{d-1} \frac{\absval{v\lrpar{\pi_j}}}{\absval{v\lrpar{\pi_k}}}}},
    \label{eq:propProbaPM}
\end{equation}
provided the values are positive. In case of zero individual values, the probabilities can be computed by considering the subgame $D\setminus Z$ with $Z = \lrcubra{j \in D~:~v(\{j\})=0}$. The interpretation of this result from \cite{Dehez2017} can be leveraged to construct a simple sampling scheme \wrt this random order distribution. First, draw $\pi_d$ from $D$ where each $j \in D$ has probability $v\lrpar{\lrcubra{j}}/\sum_{i \in D} v\lrpar{\lrcubra{i}}$, then draw $\pi_{d-1}$ from $D \setminus \pi_d$, where $j \in D$ has probability $v\lrpar{\lrcubra{j}}/\sum_{i \in D \setminus \pi_{d}} v\lrpar{\lrcubra{i}}$ and continue until every element of $D$ has been drawn. The resulting permutation $\pi = (\pi_1, \dots, \pi_d)$ will have been drawn with probability as in \Eqref{eq:propProbaPM}. This sampling scheme is further detailed in Algorithm~\ref{alg:PSPermutSampling}.

\begin{algorithm}[ht!]
\caption{Drawing permutations according to the proportional Shapley values}\label{alg:PSPermutSampling}
\begin{algorithmic}[1]
        \Require Cooperative game $(D=\{1,\dots, d\},v)$
        \Require Computed values for $v(\{1\}), \dots, v(\{d\})$
    \Ensure Permutation $\pi = (\pi_1, \dots, \pi_d)$ drawn according to \Eqref{eq:propProbaPM}
    \State Initialize $\tilde{D} \leftarrow D$
    \For{$j=0,\dots,d-1$}
        \State Draw $\pi_{d-j}$ from $\tilde{D}$ where each $k \in \tilde{D}$ has probability $v\lrpar{\lrcubra{k}}/\sum_{i \in\tilde{D}} v\lrpar{\lrcubra{i}}$
        \State Let $\tilde{D} \leftarrow \tilde{D} \setminus \pi_{d-j}$
    \EndFor
    \State \Return $\pi=\lrpar{\pi_1, \dots, \pi_d}$
\end{algorithmic}
\end{algorithm}

\subsection{Importance Sampling: Reweighting Computed Values}\label{app:subsec:impsampling}

By applying a reweighting of the summands in \Eqref{eq:approxWeberAlloc}, it is possible to produce estimates for a different random order allocation without re-sampling permutations or re-computing the value functions on selected coalitions. More precisely, let $p$ and $p'$ be two different random order distributions, related to the random order allocations $\phi$ and $\phi'$. Let $\Pi_m$ be a sample of $\calS_D$ drawn according to $p$. Then, it can be shown that
$$ \forall j \in D, \enspace \widehat{\phi'_v}^{\textrm{IS}} \eqdef \frac{1}{m} \sum_{\pi \in \Pi_m} \frac{p'(\pi)}{p(\pi)}\lrbra{v\lrpar{\pi^j} - v\lrpar{\pi^j \setminus \lrcubra{j}}},$$
is an unbiased, consistent, and asymptotically normal estimate of $\phi'$, from basic Monte Carlo-type arguments. This entails that if a practitioner commits to approximation, e.g., the Shapley values by drawing random permutations, the computations of the value function on the selected coalitions can be recycled to approximate, e.g., the proportional Shapley values by reweighting the summands in \Eqref{eq:approxWeberAlloc} by $\textrm{PS}(\pi)\times d!$.

%%===========================================================%%
%% Proofs
\section{Proofs}
\label{apdx:proofs}
\subsection{Efficiency property of CP-based UA}

%% Proof of Prop 3.1
\begin{proof}[Proof of Proposition~\ref{prop:eff:shap:cp}]
This is a direct consequence of Proposition~\ref{prop:effHarsa} as both the Shapley and proportional Shapley values are in the Harsanyi set of allocations.
\end{proof}

\subsection{Statistical guarantees of the approximation scheme}
%%===========================================================%%
%% Proof of thm 1
\begin{thm*}[Formal statement of Theorem~\ref{thm:cvgmontecarlo}]
Let $p$ be any probability mass function over the permutations of $\calS_D$. For a positive integer $m$, let $\pi_1, \dots, \pi_m$ be an i.i.d. sample drawn from $p$. For $\pi \in \calS_D$, let $h(\pi) =v(\pi^j) - v(\pi^j \setminus \{j\})$. Assume that $0<\mathbb{E}_p\lrbra{h(\pi)^2} < \infty$. Then, for every $j \in D$, 
$$\widehat{\phi_v}(j) = \frac{1}{m} \sum_{i=1}^m h(\pi_i),$$
as computed by Algorithm~\ref{alg:WeberPermut}, are unbiased, strongly consistent, and asymptotically normal estimators of $\phi_v(j) := \mathbb{E}_{p}[h(\pi)]$. 

This result specializes to the Shapley values when $p= \calU(\calS_D)$, and to the proportional Shapley values when $p=PS$ as in \Eqref{eq:propProbaPM}, corresponds to the distributions associated with the Shapley value or the proportional Shapley value.
\end{thm*}

\begin{proof}[Proof of Theorem~\ref{thm:cvgmontecarlo}]
Let \( j \in D \). The estimator
\[
\widehat{\phi_v}(j) = \frac{1}{m} \sum_{i=1}^m h(\pi_i),
\]
corresponds to the classical Monte Carlo integration estimator of \( \mathbb{E}_{p}[h(\pi)] \). Following \cite{mccasella, strumbelj2010jmlr}, provided \( \mathbb{E}_{p}[|h(\pi)|] < \infty \), the strong law of large numbers ensures that \( \widehat{\phi_v}(j) \) is strongly consistent. Moreover, assuming \( 0<\mathbb{V}_{p}\lrpar{f(\pi)}<\infty\), the central limit Theorem guarantees asymptotic normality. Finally, unbiasedness comes from the linearity of the expectation operator.
\end{proof}

%%===========================================================%%
%% Proof of thm 2
\begin{thm*}[Formal statement of Theorem~\ref{thm:cvgimpsampling}]
Let $p$ and $p'$ be two probability mass functions over $\calS_D$, and for a positive integer $m$, let $\pi_1, \dots, \pi_m$ be an i.i.d. sample drawn from $p$. Assume that $0<\mathbb{E}_{p}\lrbra{\lrpar{\frac{p'(\pi)}{p(\pi)}h(\pi)}^2} < \infty$. Then, for every $j \in D$, 
$$\widehat{\phi'_v}^{\textrm{\normalfont IS}} = \frac{1}{m} \sum_{i=1}^m \frac{p'(\pi_i)}{p(\pi_i)} h(\pi_i)$$
are unbiased, strongly consistent, and asymptotically normal estimators of $\phi_v'(j) := \mathbb{E}_{p'}[h(\pi)]$.

For every $j \in D$, if $p' = PS$ and $p = \calU(\calS_D)$, this results particularizes to $\widehat{\text{P-Shap}}^{\text{IS}}_v(j)$, and if $p' = \calU(\calS_D)$ and $p = PS$, it particularizes to $\widehat{\text{Shap}}^{\text{IS}}_v(j)$.
\end{thm*}

\begin{proof}[Proof of Theorem~\ref{thm:cvgimpsampling}]
    Let \( j \in D \). The estimator
$$\widehat{\phi'_v}^{\text{IS}}(j) = \frac{1}{m} \sum_{i=1}^m \frac{p'(\pi_i)}{p(\pi_i)} h(\pi_i),$$
corresponds to the classical importance sampling estimator of $\mathbb{E}_{p'}[h(\pi)]$. Provided $\mathbb{E}_{p}\left[\frac{p'(\pi)}{p(\pi)}h(\pi)\right] < \infty$, the strong law of large numbers ensures that $\widehat{\phi'_v}^{\text{IS}}(j)$ is strongly consistent \cite{mccasella}. Moreover, provided $0<\mathbb{V}_{p}\lrpar{\frac{p'(\pi)}{p(\pi)} h(\pi)} < \infty$, the central limit theorem guarantees asymptotic normality. Finally, unbiasedness comes from the linearity of the expectation operator.
\end{proof}

%%===========================================================%%
%% Additional results
\section{Additional Details and Results on the Experiments}\label{sec:addExps}

\subsection{Modified Sobol'-Ativan Benchmark}\label{apdx:results-sobolativan}

\Figref{apdx:fig:mc-conv} provides an alternative view of the convergence behavior of the Shapley value estimates using the sampling-based approximation method. For each variable, the figure reports the estimates obtained across 150 replications for each number of sampled permutations $m$ (y-axis). Shapley values are shown along the x-axis. Dots represent the mean estimate across replications, and horizontal bars indicate $\pm1$ standard deviation. Vertical lines mark the Shapley values computed using the exact method. As expected, the variance of the estimates---computed over 150 replications for each value of sampled permutations $m$---decreases with the number of permutations.

\begin{figure}[ht!]
    \centering
    \includegraphics[width=\linewidth]{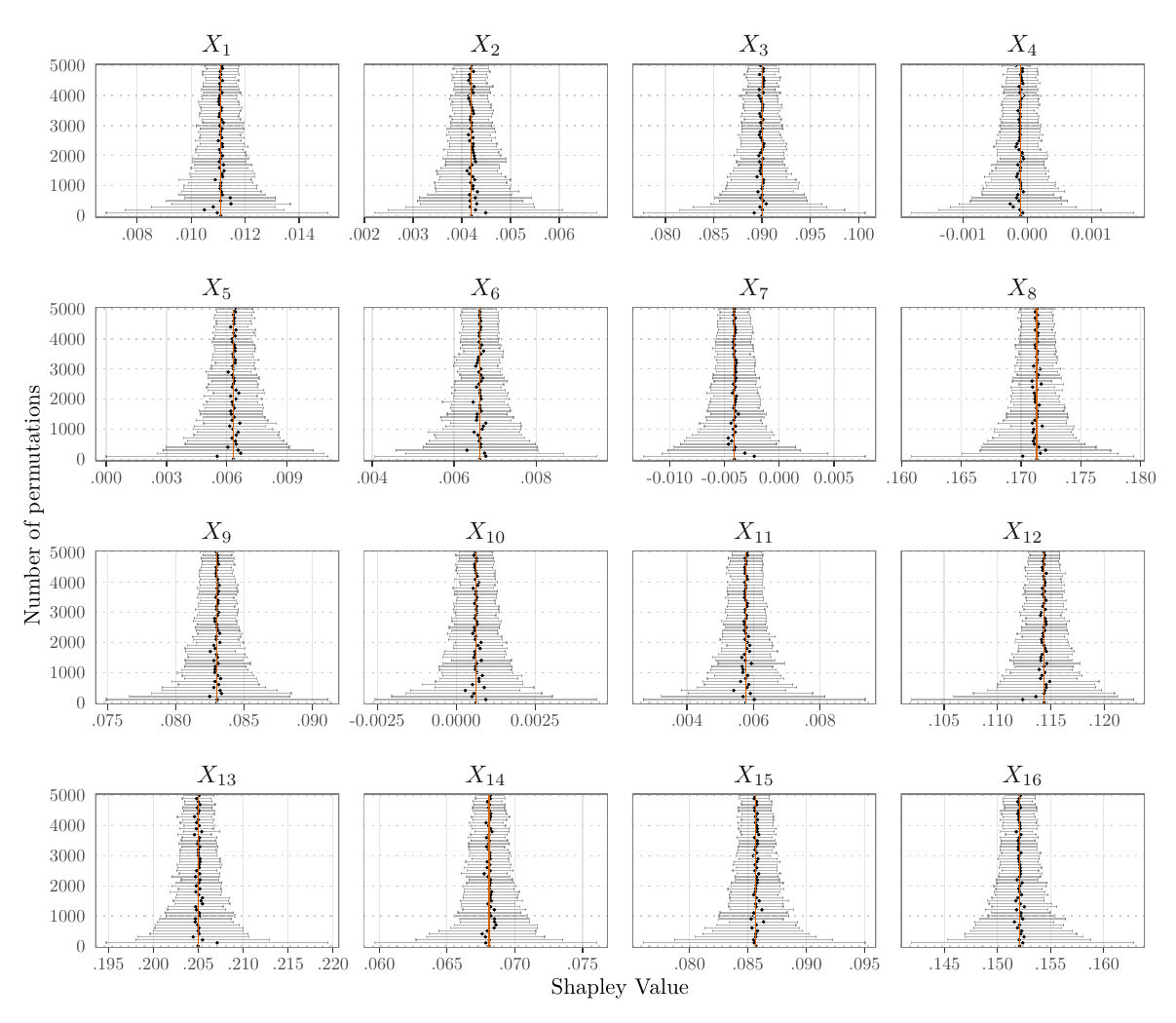}%
    \caption{Empirical convergence of Monte Carlo Shapley value estimates. For each number of permutations (y-axis), a black dot indicates the mean Shapley value over 150 replications, and horizontal bars indicate $\pm1$ standard deviation. The orange vertical line marks the Shapley value computed with the exact estimation procedure.}
    \label{apdx:fig:mc-conv}
\end{figure}

\subsection{Real-World Datasets}
We conduct a series of experiments on real-world datasets to estimate feature-wise allocations that quantify each variable's contribution to uncertainty. These allocations are based on Conformal Prediction (CP) intervals, and the underlying models used to construct the CP intervals vary across experiments. Depending on the CP method employed, different value functions are defined to reflect the target quantity (e.g., interval width or predicted bound). We then estimate allocations using either Shapley values or Proportional Shapley (P-Shapley) values. An overview is reported in \Tabref{tab:methods-overview}.

\begin{table}[h!]
\centering
\begin{tabular}{l|l|l|l}
\hline
\textbf{Method} & \textbf{Predictive Models} & \textbf{Value Functions} & \textbf{Allocation Methods} \\
\hline
SMR & \begin{tabular}[t]{@{}l@{}}LR,\\ RF,\\ LGB\end{tabular} 
    & \begin{tabular}[t]{@{}l@{}}Width,\\ Upper Bound,\\ Lower Bound,\\ Cond. Mean\end{tabular}
    & \begin{tabular}[t]{@{}l@{}}Shap,\\ P-Shap\end{tabular} \\
\hline
LACP & \begin{tabular}[t]{@{}l@{}}\textit{Cond. Mean:} LR,\\ RF,\\ LGB\\ \textit{MAD:} LGB (\textit{MAD})\end{tabular}
     & \begin{tabular}[t]{@{}l@{}}Width,\\ Upper Bound,\\ Lower Bound,\\ Cond. Mean and MAD\end{tabular}
     & \begin{tabular}[t]{@{}l@{}}Shap,\\ P-Shap\end{tabular} \\
\hline
CQR & \begin{tabular}[t]{@{}l@{}}Q-LR,\\Q-RF\end{tabular}
    & \begin{tabular}[t]{@{}l@{}}Interval Width,\\ Upper Bound,\\ Lower Bound\end{tabular}
    & \begin{tabular}[t]{@{}l@{}}Shap,\\ P-Shap\end{tabular} \\
\hline
\end{tabular}
\caption{Overview of implemented methods, models, value functions, and allocation approaches.}
\label{tab:methods-overview}
\end{table}

\paragraph{Methods, models, value functions, and allocation approaches} We evaluate our methods on the seven publicly available datasets described in \Tabref{tab:datasets}. Depending on the number of features, we either estimate the allocations using the exact method or the approximation method (\Tabref{tab:dataset-estimation}).

\begin{table}[h!]
\centering
\begin{tabular}{c|c|c|c|c}

Dataset & $n$ & $d$ & Allocation & \# Models Trained \\
\hline
\texttt{bike}     & $17,379$ & $12$ & Exact                  & $4,096$\\
\texttt{blog}     & $52,397$ & $238$& Approx.\ (\(m = 50\))  & $11,841$ \\
\texttt{casp}     & $45,730$ & $9$  & Exact & 512\\
\texttt{concrete} & $1,030$  & $8$  & Exact                  & $256$\\
\texttt{facebook} & $79,788$ & $37$ & Approx.\ (\(m = 200\)) & $6,814$\\
\texttt{UScrime}  & $1,993$   & $101$& Approx.\ (\(m = 200\)) & $19,780$\\
\texttt{star}     & $2,161$  & $38$ & Approx.\ (\(m = 200\)) & $7,220$\\
\hline
\end{tabular}
\caption{Estimation of allocations and model count by dataset.}
\label{tab:dataset-estimation}
\end{table}

\paragraph{Hyperparameter Settings} We rely on the default hyperparameters provided by the corresponding \texttt{R} packages, with a few targeted modifications. For Random Forests and Quantile Random Forests (\texttt{randomForest}, \texttt{quantregForest}), we set the minimum node size to 20\% of the training set, limit the number of trees to 75, and define the number of variables considered at each split (\texttt{mtry}) as $\sqrt{d}$. For LightGBM (\texttt{lightgbm}), we fix the number of boosting rounds to 100, except for the \texttt{USCrimes}, \texttt{star}, and \texttt{blog} datasets, where it is fixed to 25. In the case of Linear Quantile Regression (\texttt{quantreg}), we use the Frisch–Newton algorithm for optimization and set the tolerance to the default \texttt{R} precision, \texttt{.Machine\$double.eps}.

\paragraph{Visualization of Feature-Wise Contribution to Uncertainty} For each dataset, after estimating the allocations corresponding to a given value function and CP method flavor, we visualize the results using two complementary representations. The first is a matrix plot with variables as rows and ranks as columns. The cell $[i,j]$ indicates the percentage of test-set observations for which the absolute value of the allocation (Shapley or P-Shapley) for variable $i$ ranks in position $j$ among all variables. The cell color encodes this percentage: the darker the green, the more frequently variable $i$ appears in rank $j$ across the test set. The second visualization highlights the most frequently top-ranked variables. For each rank from 1 to 5 (x-axis), we plot the proportion of test-set observations for which a given variable appears most often in that position when allocations are ordered by absolute value. The y-axis reflects the frequency of the dominant variable at each rank. This shows which features consistently contribute the most to uncertainty. It is important to note that a single variable may appear as the most frequent at multiple ranks, reflecting consistent influence across several positions in the allocation rankings.

\newcommand{\footerplotranks}[1]{
Heatmap of variable ranks across test observations (\textbf{A}), and top-5 rank frequencies of the most dominant variables (\textbf{B}). Top to bottom CP intervals: SMR, LACP, CQR. Value function: width. \texttt{#1} dataset
}

\newpage
%% Bike dataset
\subsubsection{\texttt{bike} Dataset}
\begin{figure}[ht!]
    \centering
    \includegraphics[width=\linewidth]{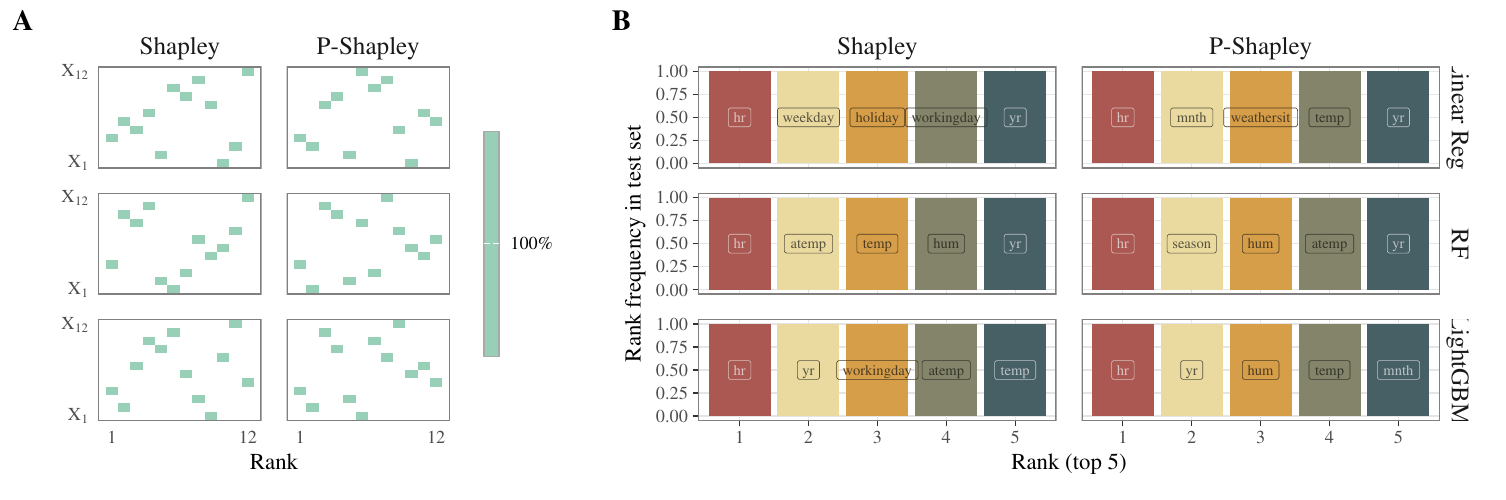}\\
    \includegraphics[width=\linewidth]{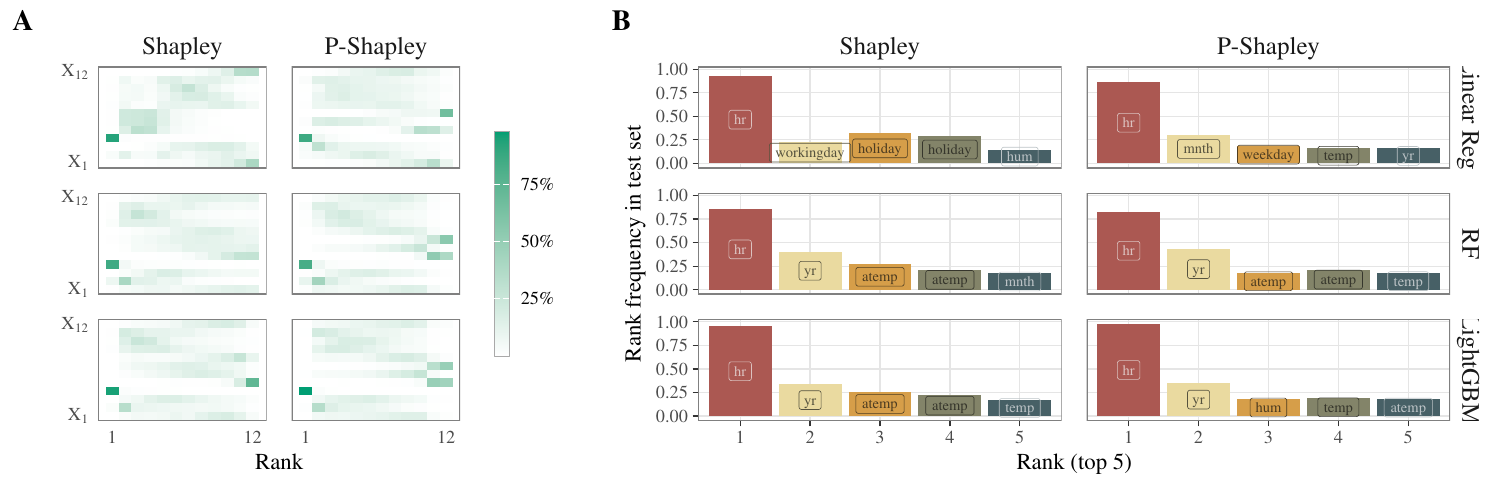}\\
    \includegraphics[width=\linewidth]{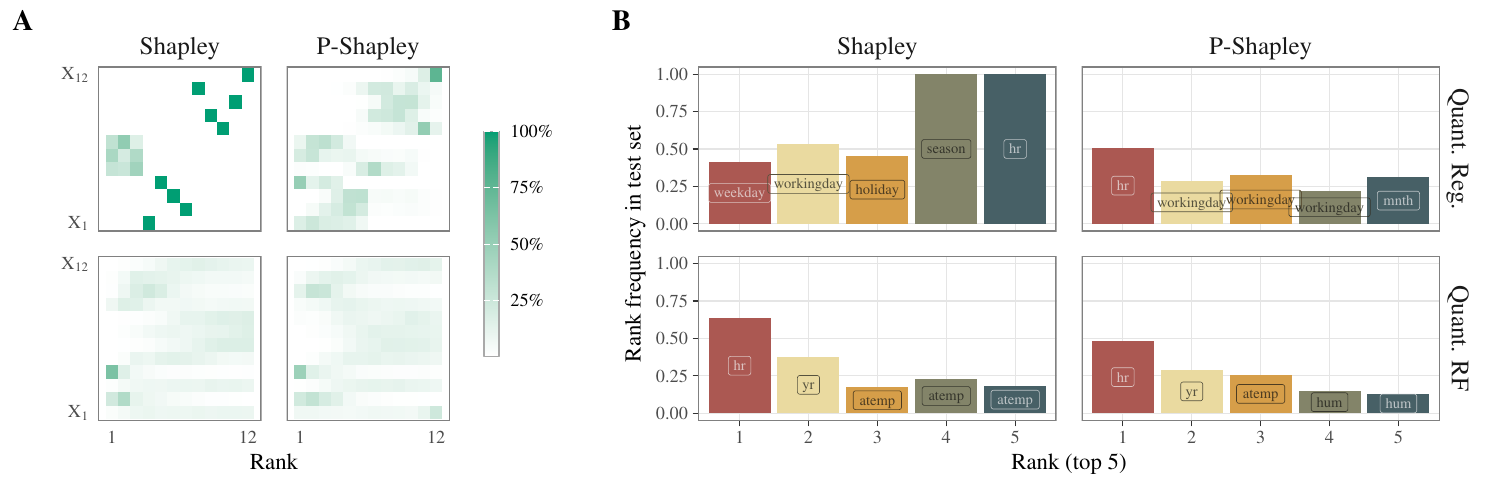}
    \caption{\footerplotranks{bike}}
    \label{fig:top-ranks-bike}
\end{figure}
\

\newpage
% Blog dataset
\subsubsection{\texttt{blog} Dataset}
\begin{figure}[ht!]
    \centering
    \includegraphics[width=\linewidth]{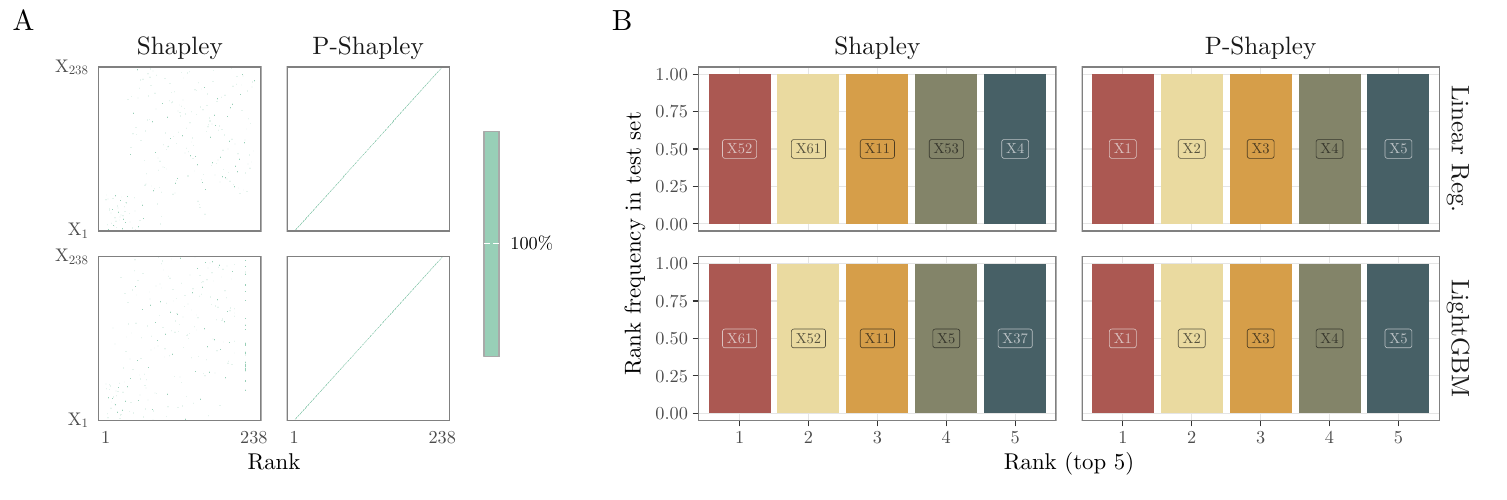}\\
    \includegraphics[width=\linewidth]{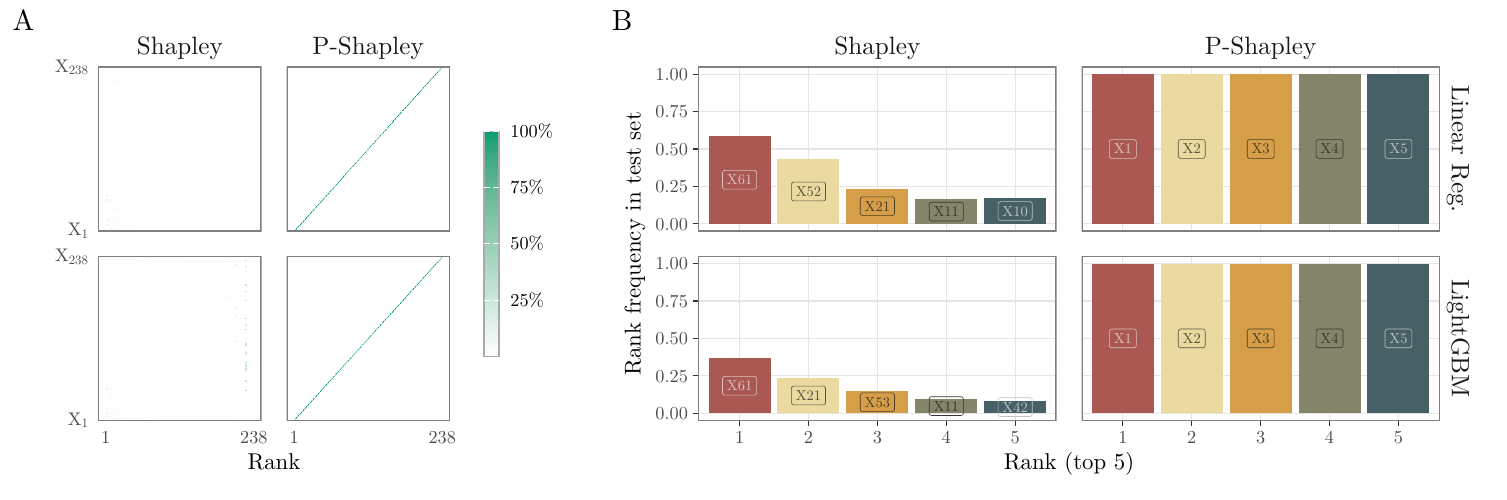}\\
    \includegraphics[width=\linewidth]{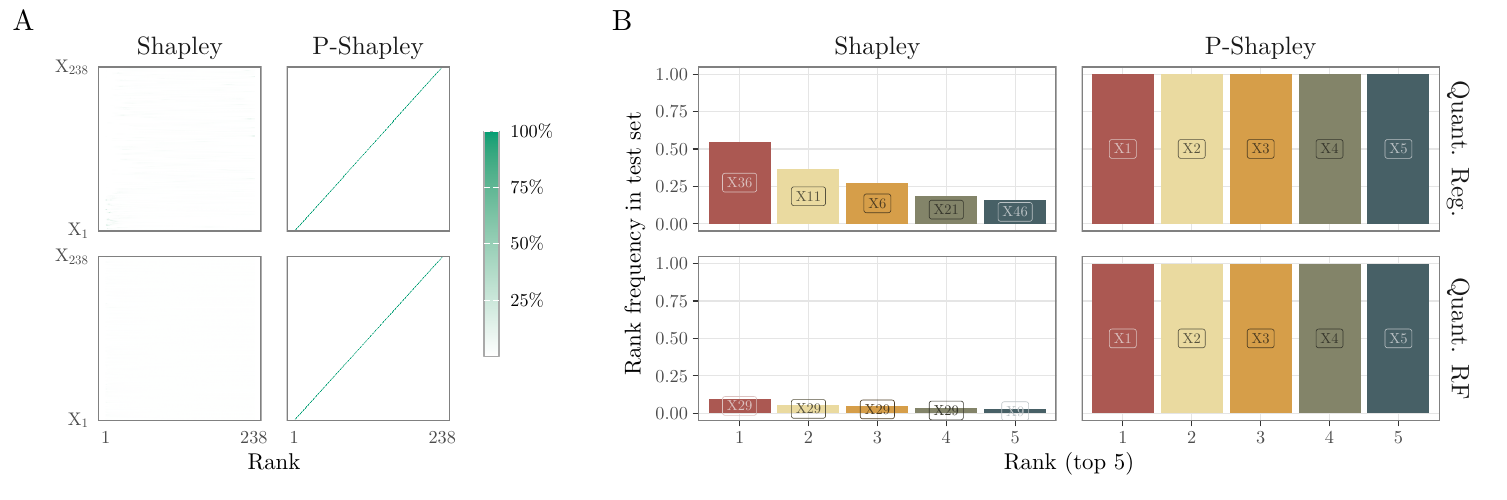}
    \caption{\footerplotranks{blog}}
    \label{fig:top-ranks-blog}
\end{figure}
\

\newpage
%% casp dataset
\subsubsection{\texttt{casp} Dataset}
\begin{figure}[ht!]
    \centering
    \includegraphics[width=\linewidth]{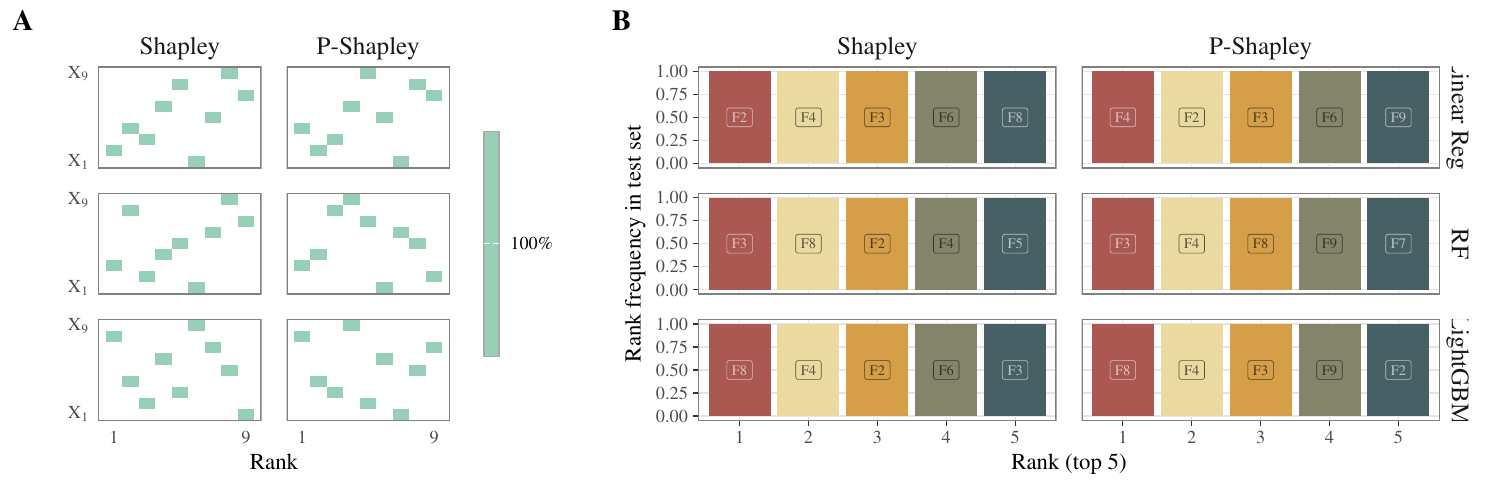}\\
    \includegraphics[width=\linewidth]{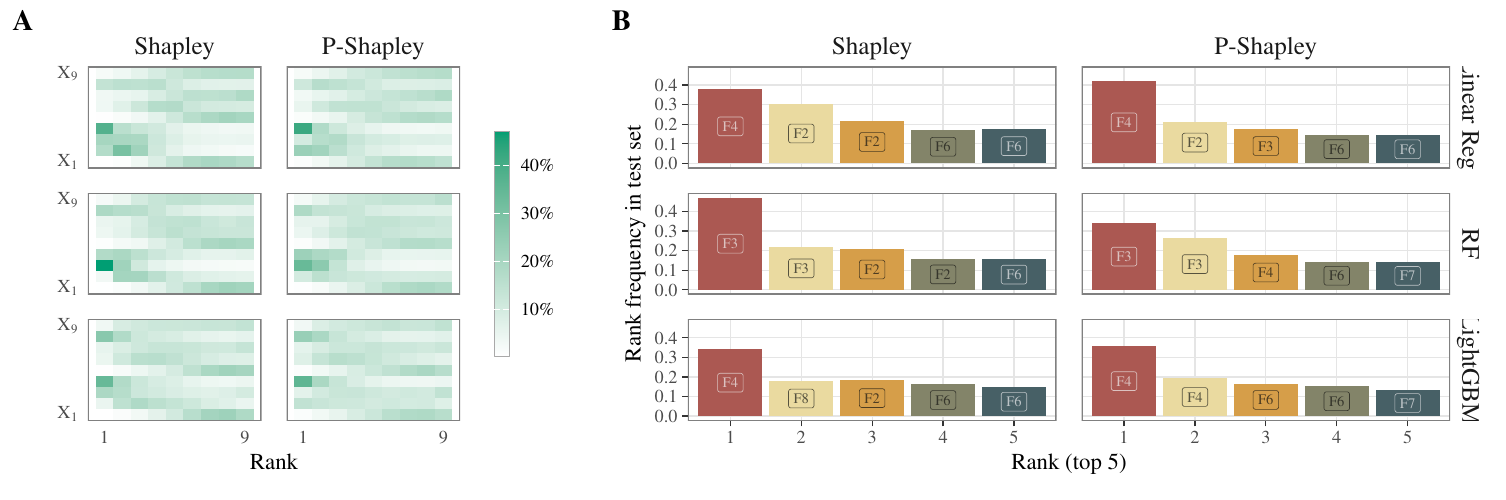}\\
    \includegraphics[width=\linewidth]{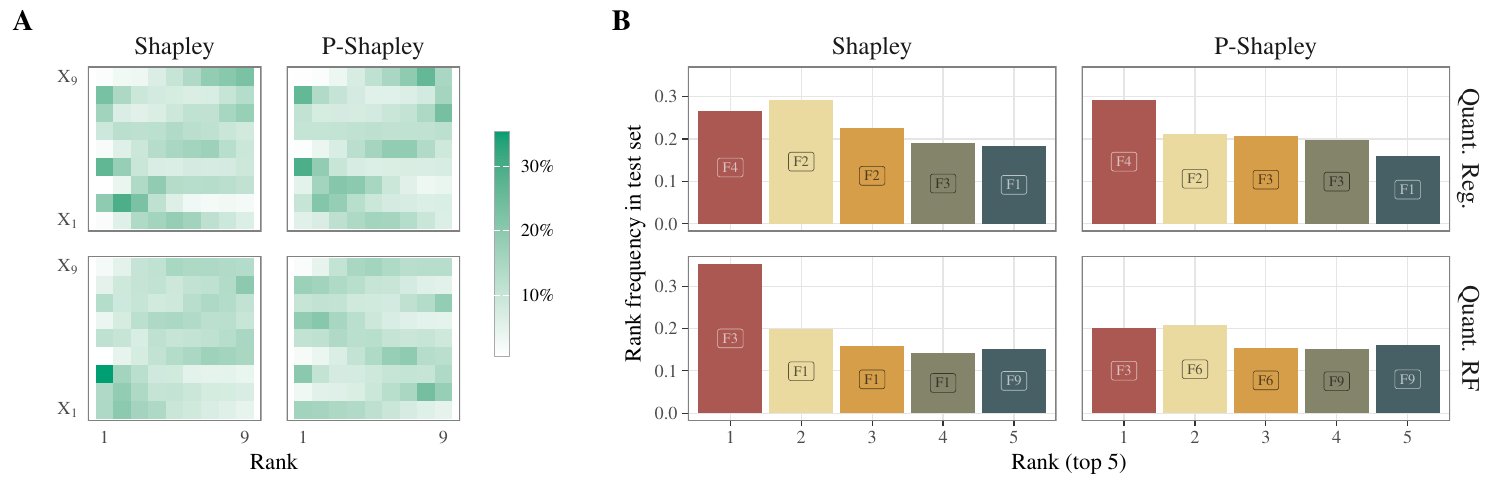}
    \caption{\footerplotranks{casp}}
    \label{fig:top-ranks-casp}
\end{figure}
\

\newpage
%% Concrete dataset
\subsubsection{\texttt{concrete} Dataset}
\begin{figure}[ht!]
    \centering
    \includegraphics[width=\linewidth]{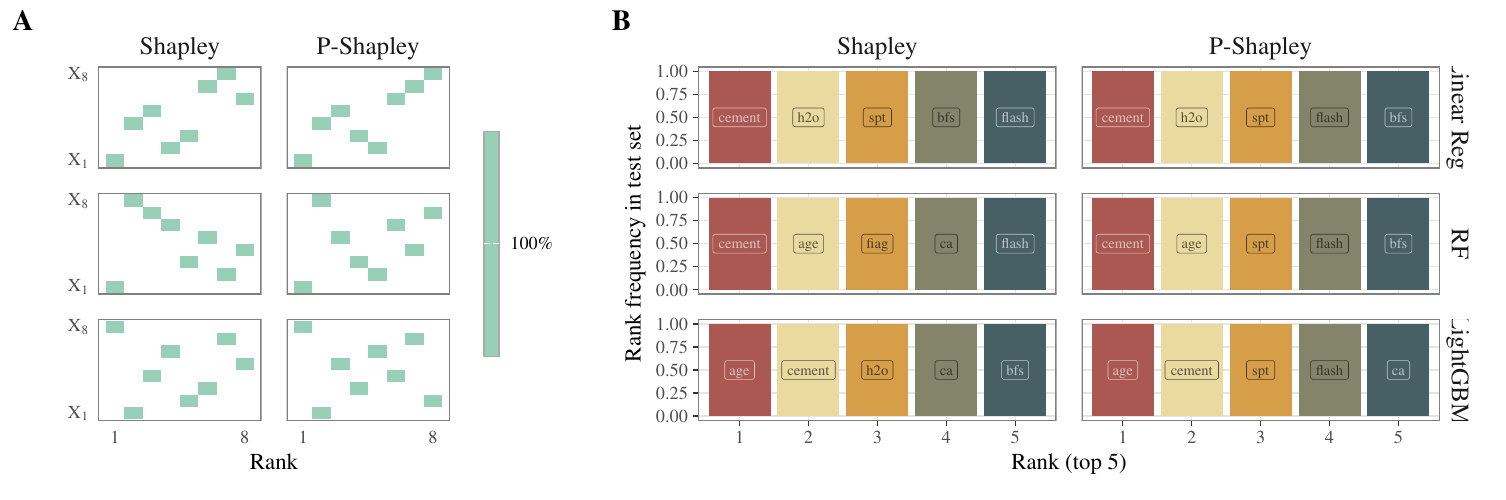}\\
    \includegraphics[width=\linewidth]{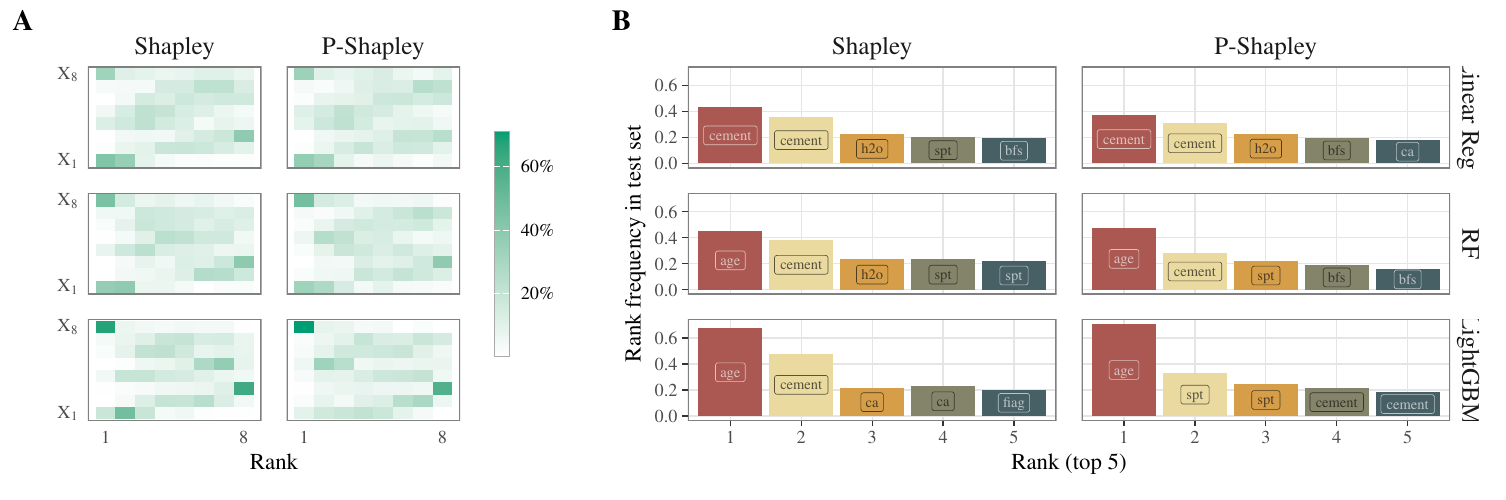}\\
    \includegraphics[width=\linewidth]{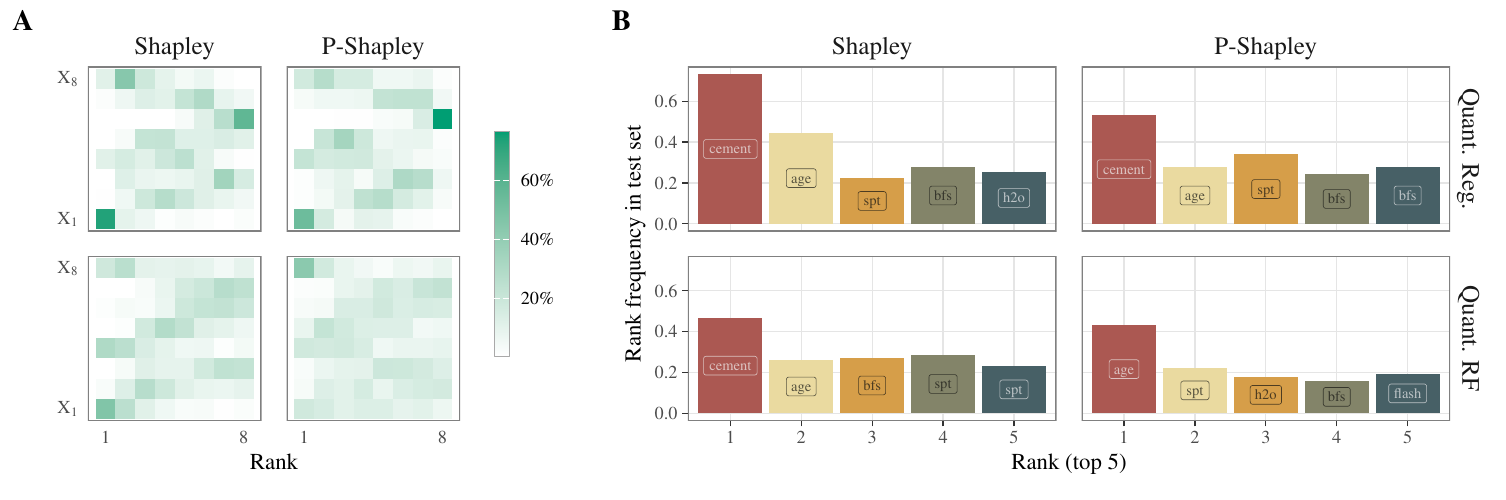}
    \caption{\footerplotranks{concrete}}
    \label{fig:top-ranks-concrete}
\end{figure}
\

\newpage
%% Facebook dataset
\subsubsection{\texttt{facebook} Dataset}
\begin{figure}[ht!]
    \centering
    \includegraphics[width=\linewidth]{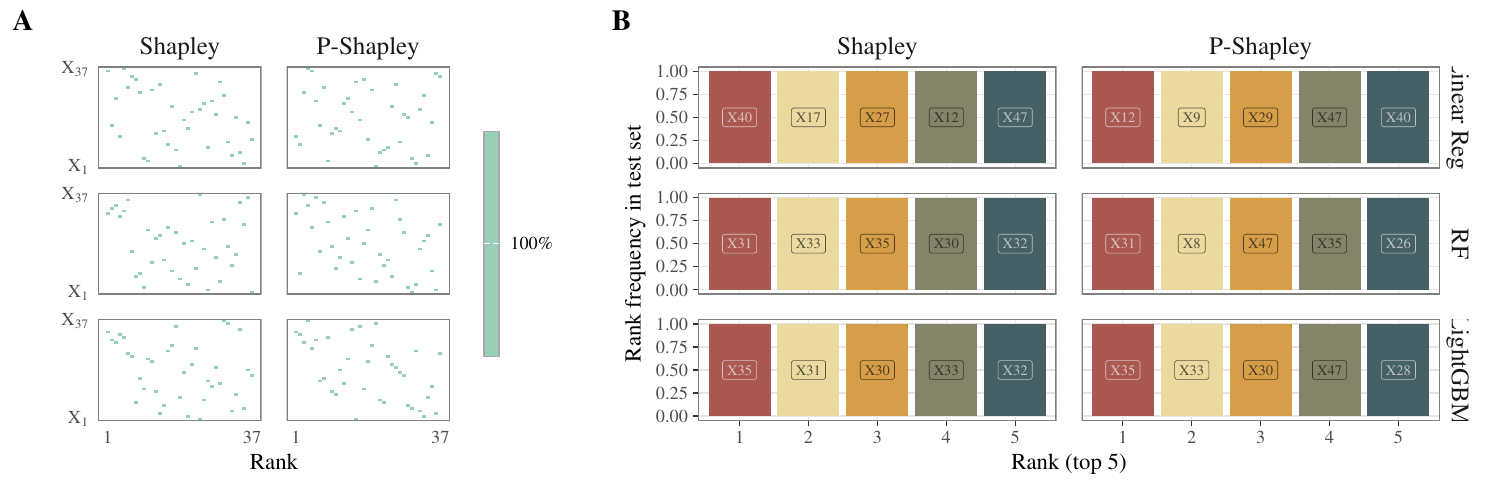}\\
    \includegraphics[width=\linewidth]{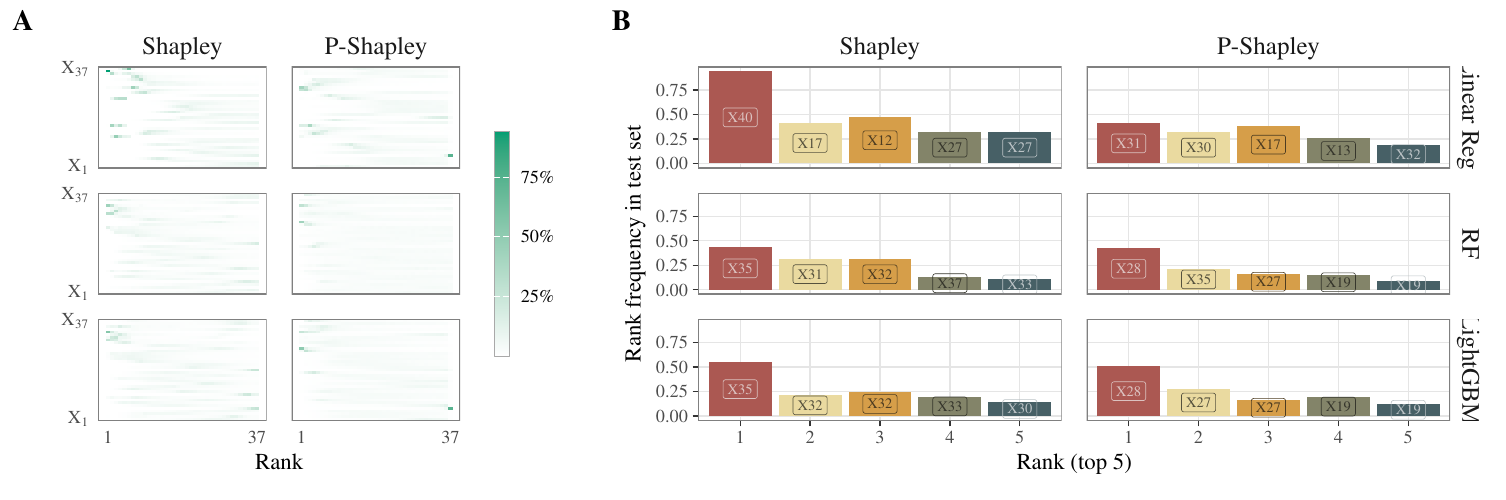}\\
    \includegraphics[width=\linewidth]{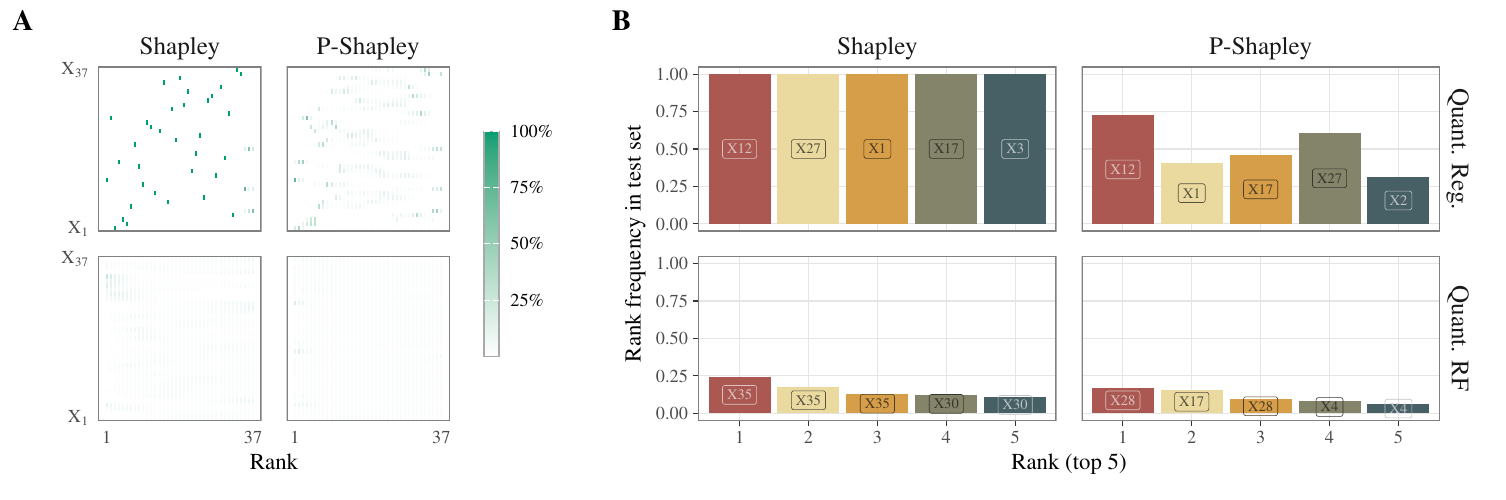}
    \caption{\footerplotranks{facebook}}
    \label{fig:top-ranks-facebook}
\end{figure}
\

\newpage
%% UScrime dataset
\subsubsection{\texttt{UScrime} Dataset}
\begin{figure}[ht!]
    \centering
    \includegraphics[width=\linewidth]{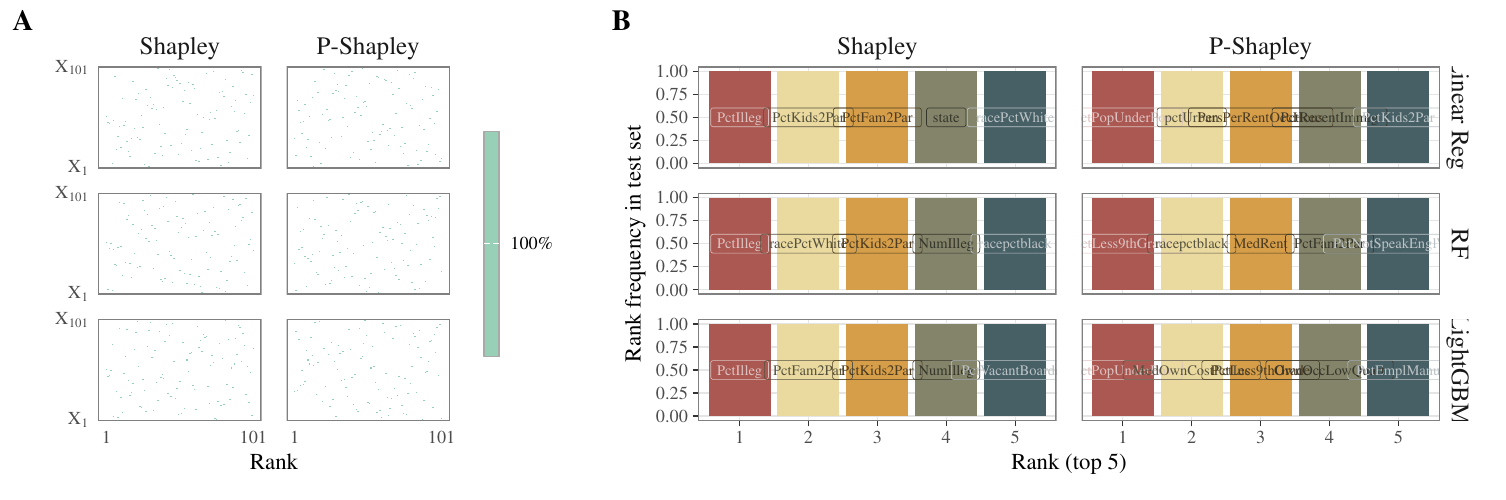}\\
    \includegraphics[width=\linewidth]{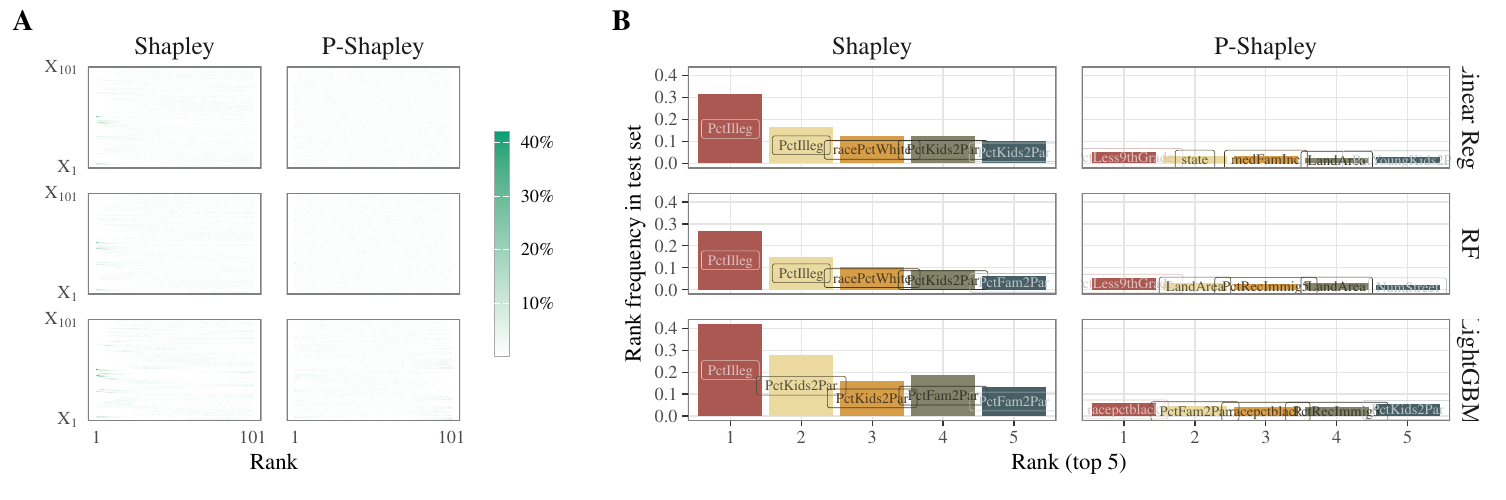}\\
    \includegraphics[width=\linewidth]{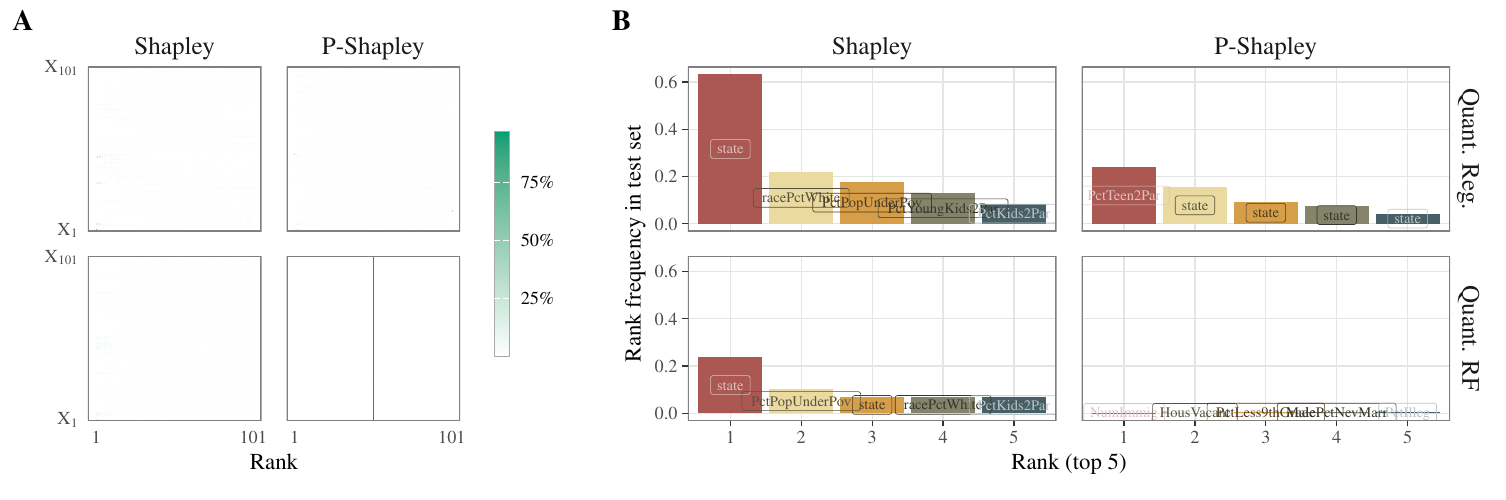}
    \caption{\footerplotranks{uscrime}}
    \label{fig:top-ranks-uscrime}
\end{figure}
\

\newpage
%% Star dataset
\subsubsection{\texttt{star} Dataset}
\begin{figure}[ht!]
    \centering
    \includegraphics[width=\linewidth]{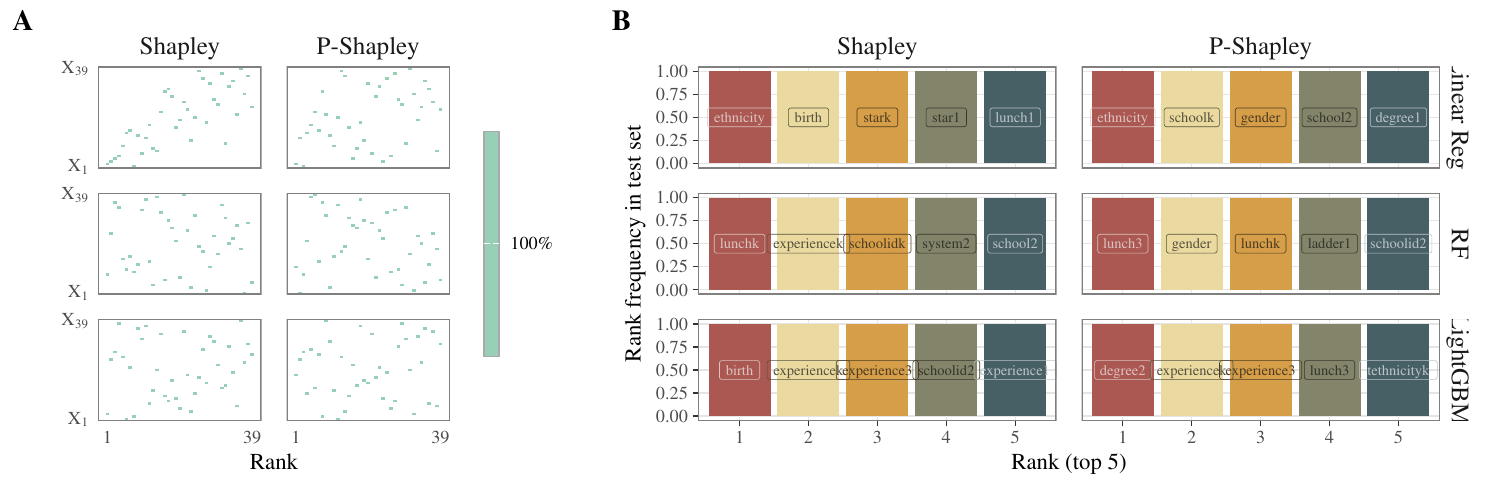}\\
    \includegraphics[width=\linewidth]{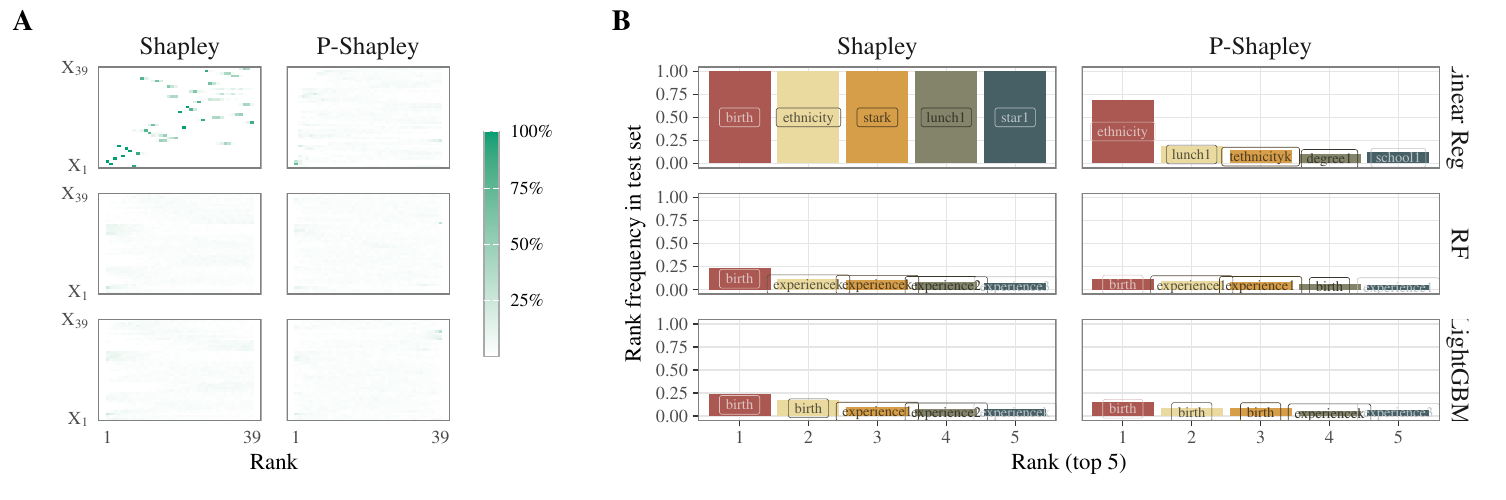}\\
    \includegraphics[width=\linewidth]{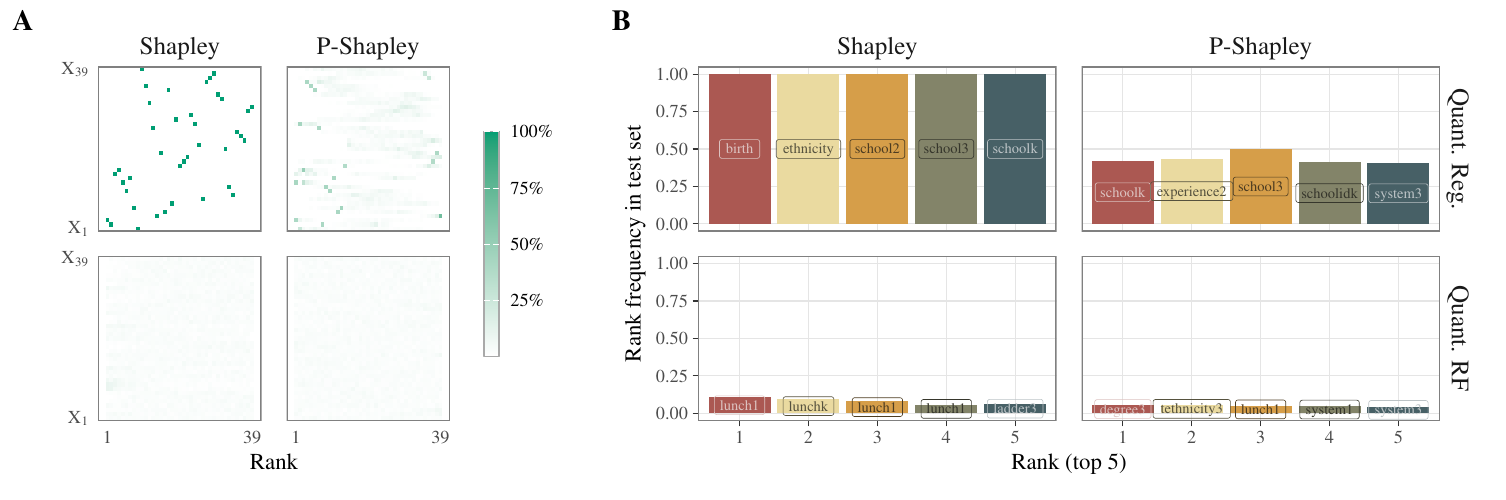}
    \caption{\footerplotranks{star}}
    \label{fig:top-ranks-star}
\end{figure}
\

\end{document}